\documentclass[USenglish,onecolumn,12pt]{article}

\usepackage[utf8]{inputenc}%(only for the pdftex engine)
\usepackage[big]{dgruyter}

\usepackage{amsmath,amssymb,amsthm,mathtools,bbm}
\usepackage{microtype}
\usepackage{enumitem}
\usepackage{hyperref}
\hypersetup{
  colorlinks=true,
  urlcolor=magenta,
  linkcolor=black,
  citecolor=black
}
\usepackage{bm}

\usepackage{threeparttable}
\usepackage{booktabs}
\usepackage{rotating}

\usepackage{algorithm}
\usepackage{algpseudocode}

\makeatletter
\providecommand{\plist@algorithm}{Algorithm~}
\makeatother

\makeatletter
\newtheorem*{rep@theorem}{\rep@title}
\newcommand{\newreptheorem}[2]{%
\newenvironment{rep#1}[1]{%
 \def\rep@title{#2 \ref{##1}}%
 \begin{rep@theorem}}%
 {\end{rep@theorem}}}
\makeatother

\newtheorem{assumption}{Assumption}
\newtheorem{theorem}{Theorem}
\newtheorem{lemma}{Lemma}
\newtheorem{proposition}{Proposition}

\newreptheorem{proposition}{Proposition}
\newreptheorem{theorem}{Theorem}
\newreptheorem{corollary}{Corollary}
\newreptheorem{lemma}{Lemma}

\begin{document}

  \articletype{}

  \author*[1]{Eric V. Strobl}

\affil[1]{Department of Biomedical Informatics, University of Pittsburgh, Pittsburgh, PA 15213, USA}

  \title{\huge Fast Nonparametric Conditional Independence Testing via Two-Stage Regression}

  \runningtitle{Two-stage regression for CI testing}

  %\subtitle{...}

  \begin{abstract}
{Constraint-based causal discovery relies on repeated conditional independence tests, but fast nonparametric tests often sacrifice calibration, especially when variables depend on the conditioning set through nonlinear relationships. We introduce BLITZ (Broad-to-Local Independence Testing via residualiZation), a nonparametric conditional independence test designed to run well under a second while maintaining the accuracy needed for the thousands of queries performed by constraint-based causal discovery algorithms. BLITZ first removes broad smooth dependence on the conditioning set using low-order polynomial regression, then applies a small nonlinear feature map and residualizes those features with shallow tree regressions. The resulting statistic tests residual cross-covariance, with a moment-matched chi-square approximation to the null distribution. We show theoretically that the two-stage design reduces the effective complexity faced by the tree residualizers, allowing shallow trees to control residual conditional-mean bias while avoiding excessive overfitting. In simulations, BLITZ provides better null calibration than fast kernel, random-feature, and regression-based competitors while remaining among the fastest methods tested. In causal discovery experiments on synthetic graphs and flow-cytometry data, BLITZ yields more reliable endpoint orientations among retained adjacencies and competitive structural recovery. These results suggest that broad-to-local residualization is a practical route to calibrated, scalable nonparametric conditional independence testing for causal discovery.}
\end{abstract}
  % \classification[MSC]{Please put MSC 2010 codes here.}
 % \communicated{...}
 % \dedication{...}

%  \journalname{Journal of Causal Inference}
%\DOI{DOI}
%  \startpage{1}
%  \received{..}
%  \revised{..}
%  \accepted{..}

%  \journalyear{2019}
%  \journalvolume{1}
%  \journalissue{1}

\maketitle

\section{Introduction}

Causal discovery seeks to recover causal relations from data. Existing approaches broadly fall into three families: constraint-based methods, score-based methods, and functional-model-based methods \cite{Glymour19}. Score-based methods search over candidate graphs by optimizing a goodness-of-fit criterion, while functional-model-based methods identify causal structure by imposing asymmetries on the data-generating mechanism. Constraint-based methods take a different route: they infer causal structure from the conditional independence (CI) relations implied by the graph. This perspective is especially attractive because, in principle, it requires fewer parametric or functional assumptions than competing approaches \cite{Huang18}. In practice, however, this advantage depends on the availability of fast, reliable nonparametric CI tests that run well under one second.
Algorithms such as PC and FCI may require thousands of CI queries, so even modest per-test costs can quickly become a computational bottleneck \cite{Shiragur24}.

The reliability of CI tests affects two distinct stages of constraint-based discovery: skeleton recovery and endpoint orientation. During skeleton search, the algorithm begins with a dense graph and removes an adjacency whenever it finds a separating set \cite{Spirtes00}. In the finite sample sizes encountered in practice, especially in nonparametric settings where reliable detection often requires substantially larger samples than parametric methods, an anti-conservative CI test can appear advantageous for adjacency recovery. By rejecting conditional independence more often, such a test removes fewer edges and may therefore preserve more true adjacencies.

This apparent skeleton-level advantage usually does not translate into reliable orientations, even though orientations are often the primary objects of scientific interest. Endpoint marks are determined not only by the recovered skeleton, but also by the separating sets recorded during skeleton search. If a CI test is anti-conservative, it often rejects true conditional independencies at smaller valid separating sets, causing the search to continue to larger conditioning sets. As conditioning sets grow, power typically deteriorates, so the test becomes more likely to eventually accept an independence for statistical rather than graphical reasons. These distorted separating sets then directly affect v-structure discovery. For an unshielded triple \(X \mathbin{*\!\! - \!\! *} Z \mathbin{*\!\! - \!\! *} Y\), constraint-based algorithms orient \(X \mathbin{*\!\!\to} Z \mathbin{\leftarrow\!\!*} Y\) only when \( Z \) is absent from the separating set for \(X\) and \(Y\) \cite{Ramsey06}. If the accepted separating set is unnecessarily large, it is more likely to include the middle node \( Z \), preventing correct collider orientation. Such errors can then propagate through subsequent orientation rules, producing incorrect endpoint marks even when the skeleton appears plausible. Thus, calibrated CI tests are particularly important for recovering the separating sets that make endpoint orientation reliable.

Unfortunately, calibration and computational efficiency are often in tension in CI testing. Valid \(p\)-values require the test to account for the information in the conditioning set, often through flexible models that can explain complex dependence between the target variables and the covariates. These models can be computationally expensive and, in finite samples, prone to overfitting and reduced power. Indeed, without restrictions on the data-generating distribution, fully valid CI tests actually have no power against any alternative \cite{Shah20}. Most fast CI tests therefore trade calibration for speed and power by using simple models \cite{Strobl19,Chalupka18}, but these approximations then fail to explain away enough conditional structure, yielding anti-conservative \(p\)-values even with only a few thousand samples. We therefore need a test fast enough for thousands of CI queries while still modeling rich conditional relationships well enough to remain calibrated for nearly all relations encountered in practice.

In this paper, we introduce BLITZ (\textbf{B}road-to-\textbf{L}ocal \textbf{I}ndependence \textbf{T}esting via residuali\textbf{Z}ation), a fast, calibrated nonparametric CI test designed for repeated use inside constraint-based causal discovery algorithms. BLITZ uses a two-stage residualization strategy to remove dependence on the conditioning set \(\bm Z \). The first stage applies a lightweight global residualizer, such as low-order polynomial regression, to remove broad smooth structure from \(X\) and \(Y\). The second stage further residualizes the first-stage residuals using a shallow tree regression. Conditional independence is then tested through the cross-covariance of the resulting second-stage residuals. We show that this broad-to-local design lets the polynomial stage remove global structure that trees would otherwise require many leaves to approximate, thereby substantially reducing the tree size needed in the second stage. The tree stage can then remain shallow and fast while spending its limited adaptivity on residual localized nonlinearities and thresholds. Empirically, BLITZ achieves calibration, speed, and power in many problem settings encountered in practice, leading to substantially more reliable endpoint orientations during causal discovery.

\section{Related Work}

CI testing spans a broad spectrum of methods, from classical parametric procedures to fully nonparametric tests. Methods that rely on strong parametric assumptions can often be very fast. The classical example is Fisher's \(z\)-test, which reduces CI testing to testing whether a partial correlation is zero under Gaussian assumptions \cite{Kalisch07}. Computationally, this requires only entries of the sample covariance matrix and inversion of the covariance submatrix corresponding to the conditioning set. By contrast, most broadly applicable nonparametric CI tests remain computationally burdensome because they rely on large kernel matrices, repeated resampling, or expensive local computations \cite{Runge18, Doran14}. Since this paper focuses on \textit{fast}, sub-second nonparametric CI testing, we do not attempt an exhaustive review of the broader nonparametric literature.

Among the faster nonparametric alternatives, RCIT and RCoT remain the clearest baselines \cite{Strobl19}. Both methods start from the kernel-based conditional independence test (KCIT) framework \cite{Zhang11} but replace exact kernel operations with random Fourier feature approximations, thereby reducing the cost of nonlinear CI testing substantially while retaining the basic kernel-testing perspective. In practice, they scale approximately linearly with sample size and were introduced explicitly to make nonparametric causal discovery feasible on larger datasets, where KCIT would otherwise be too slow to use routinely.

A second line of work accelerates CI testing by replacing kernel machinery with predictive or regression-based surrogates. FCIT adopts this strategy by testing whether adding the candidate variable to the conditioning set yields a statistically significant improvement in out-of-sample prediction \cite{Chalupka18}; this design makes the method simple, scalable, and attractive in large-sample settings, although it is ultimately heuristic because conditional dependence need not always manifest as improved conditional-mean prediction. CCI follows a related residual-based philosophy: it first removes the effect of the conditioning set and then searches for remaining nonlinear dependence through a truncated collection of basis-function correlations among residuals \cite{Ramsey14}. These methods are appealing because they are lightweight and easy to deploy, but they generally trade some of the distributional generality of kernel tests for computational speed.

FastKCI occupies a useful middle ground between these lightweight surrogates and classical kernel CI tests \cite{Schacht25}. Rather than abandoning the kernel CI statistic, it partitions the conditioning space using a mixture-of-experts model, runs local kernel CI tests on the resulting blocks in parallel, and aggregates the blockwise statistics through importance weighting. This design preserves more of the original kernel-testing logic, while still achieving substantial speedups relative to KCIT. At the same time, FastKCI remains more computationally involved than the lightest random-feature or regression-based methods, so it is best viewed as a fast kernel method rather than a truly ultralight CI test.

Despite their computational advantages, RCIT, RCoT, FCIT, CCI, and FastKCI share a recurring practical limitation: imperfect Type I error control. In finite samples, these methods can become anti-conservative or poorly calibrated when \(X\) and \(Y\) depend on the set \(\bm Z \) in complex ways, or when larger sample sizes make small residual conditioning errors detectable. Figure~\ref{fig:histograms} illustrates this behavior through null \(p\)-value histograms. The central difficulty is that these methods often leave enough residual dependence on \(\bm Z\) to shift the association statistic away from its null center. Adding random Fourier features or reducing a kernel bandwidth can absorb more \(\bm Z\)-induced structure, but only by substantially expanding the effective function class. This added flexibility increases variance, raises computational cost, and can reduce power by overfitting residual signal.

\begin{figure}[!t]
    \centering
    \includegraphics[width=0.8\linewidth]{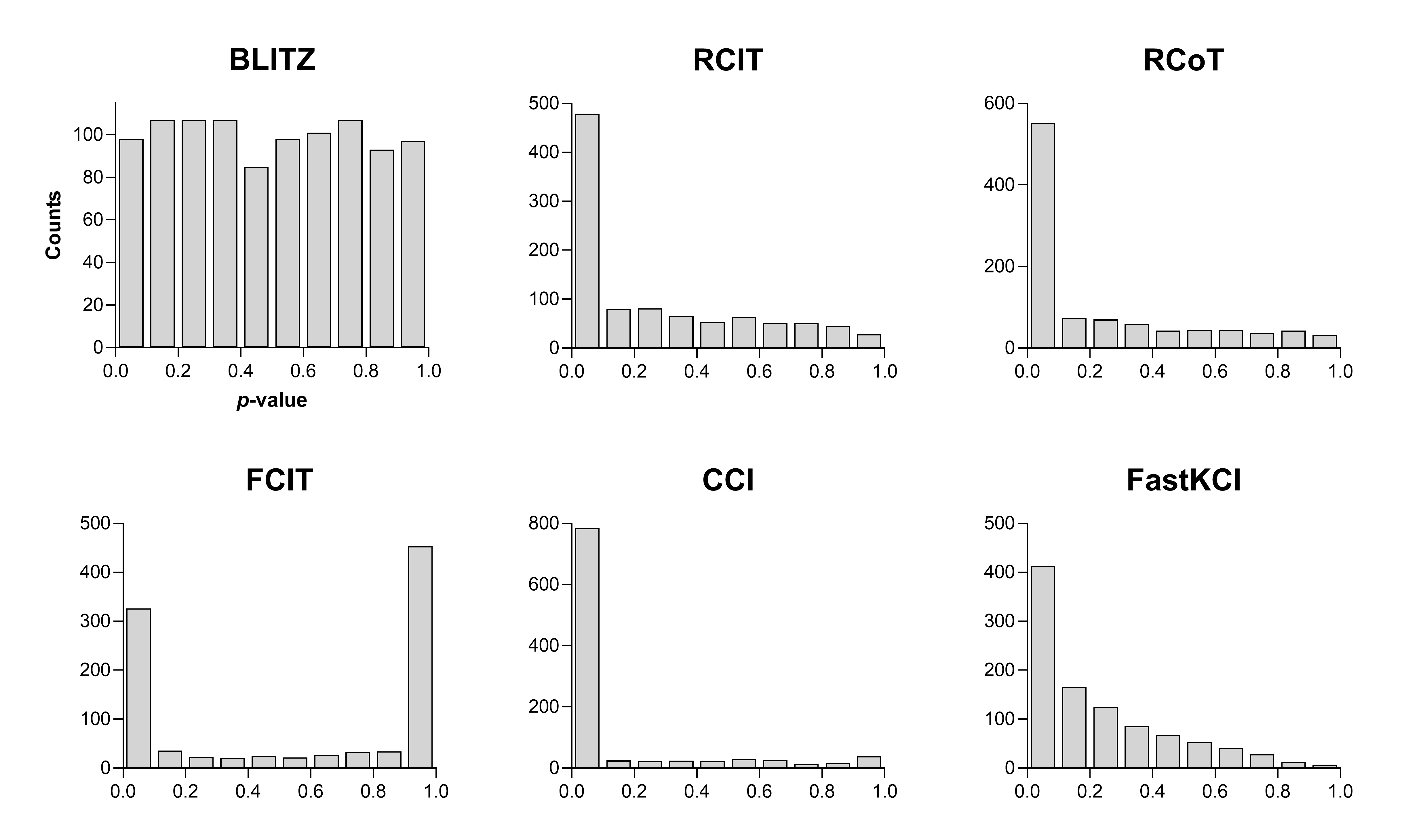}
\caption{\textbf{Representative null \(p\)-value histograms across CI tests.} All histograms share the same x- and y-axis labels. Calibration was evaluated under the nonlinear null settings described in the experiments. The histograms shown here correspond to a sample size of 5000 and a conditioning set size of 3. BLITZ produces an approximately uniform null \(p\)-value distribution, whereas RCIT, RCoT, CCI, and FastKCI show excess mass near zero, indicating anti-conservative behavior. The FCIT histogram is also visibly U-shaped, suggesting additional miscalibration beyond simple inflation of small \(p\)-values.}
    \label{fig:histograms}
\end{figure}

As discussed before, miscalibration in turn affects both adjacency recovery and orientation, but its consequences are especially damaging for endpoint marks. Anti-conservative tests may preserve more true adjacencies by rejecting independence too often, yet the same behavior can distort the separating sets used to orient \(v\)-structures. Thus, a method may return a plausible skeleton while still producing unreliable edge orientations, even though orientation is often the primary reason scientists use causal discovery methods. We address this problem by developing BLITZ, a fast and well-calibrated CI test designed for the smooth, moderately structured conditional relationships frequently encountered in constraint-based search. We show that the method controls residual dependence on \(\bm Z \) for broad classes of H\"older-smooth regression targets, while still respecting the impossibility of fully assumption-free CI testing with power against all alternatives \cite{Shah20}.

\section{BLITZ Algorithm}

We propose a fast nonparametric CI test using a two-stage residualization approach. The main idea is to use a simple global model to remove broad, smooth dependence on \(\bm Z \), and then use a flexible but shallow local model to remove the remaining residual structure. This division of labor lets the second stage focus its limited adaptivity on localized nonlinearities rather than spending many degrees of freedom approximating global trends.

We summarize the corresponding BLITZ CI test in Algorithm~\ref{alg:blitz}. The method begins with the first-stage residualization step. When \(\bm Z \) is available, BLITZ fits low-order polynomial regressions of \(X\) and \(Y\) on \(\bm Z \) and replaces both variables by their residuals in Lines \ref{alg:poly1}--\ref{alg:poly2}. This stage removes broad smooth dependence on the conditioning variables before the adaptive tree stage is used.

Line \ref{alg:transforms} then constructs the features used in the residual covariance test. Let \(\Phi_X\) and \(\Phi_Y\) denote the finite nonlinear feature maps applied to the standardized first-stage residuals. In our implementation, each coordinate \(u\) is mapped to the two swish-type features \cite{Ramachandran17}
\[
\phi_-(u)=-u\sigma(-u), \qquad \phi_+(u)=u\sigma(u),
\]
where \(\sigma(u)=(1+\exp(-u))^{-1}\). We chose these functions because their span is
\(
\operatorname{span}\{\phi_-,\phi_+\}
=
\operatorname{span}\{u,\;u\tanh(u/2)\},
\) where \(u\tanh(u/2)\) behaves like a smooth absolute-value feature: near zero,
\(
u\tanh(u/2) \approx \frac{u^2}{2},
\)
whereas for large \(|u|\),
\(
u\tanh(u/2) \approx |u|.
\)
Thus the dictionary preserves the linear residual direction while adding a smooth magnitude-like component, allowing BLITZ to detect asymmetric and nonlinear residual associations at little computational cost.

The second-stage residualization block removes remaining dependence on \(\bm Z \) from the transformed features (Lines \ref{alg:tree1}--\ref{alg:tree3}). BLITZ fits trees separately on the \(X\)- and \(Y\)-sides, with tree-size parameters selected by out-of-bag prediction. If \(\bm Z \) is empty, this step reduces to ordinary centering.

Finally, Algorithm~\ref{alg:blitz} forms the residual cross-covariance matrix \(\widehat C\) and computes \(S=n\|\widehat C\|_F^2\). Large values of \(S\) indicate residual association between the transformed \(X\)- and \(Y\)-side features after both stages of residualization. The null distribution is approximated by a moment-matched weighted chi-square law using the empirical eigenvalues of the product-feature covariance matrix. 

Finally, the computational complexity analysis in Supplementary Materials~\ref{supp:complexity} shows that, for fixed conditioning-set size and fixed tuning parameters, BLITZ has quasilinear \(O(n\log n)\) scaling in the sample size. A fast R implementation with C++ back-end is available at \url{https://github.com/ericstrobl/BLITZ}.

\begin{algorithm}[!h]
\caption{BLITZ: Broad-to-Local Independence Testing via Residualization}
\label{alg:blitz}
\begin{algorithmic}[1]
\Require Observations \(\{(X_i,Y_i,Z_i)\}_{i=1}^n\), polynomial degree \(r\), tree tuning grid \(\mathcal G\)
\Ensure Test statistic \(S\) and \(p\)-value

\If{\(\bm Z \) is nonempty}
    \State Fit separate polynomial regressions of \(X\) on \(\bm Z\) and \(Y\) on
\(\bm Z\), including interactions up to degree \(r\). \label{alg:poly1}
    \State Form first-stage residuals
    \(
    \widetilde X_i = X_i-\widehat m_X(Z_i)\) and \(
    \widetilde Y_i = Y_i-\widehat m_Y(Z_i).
    \) \label{alg:poly2}
\Else
    \State Set \(\widetilde X=X\) and \(\widetilde Y=Y\).
\EndIf

\State Standardize \( \widetilde X \) and \( \widetilde Y \) to unit variance.

\State Construct nonlinear feature matrices \(G_X=\Phi_X(\widetilde X)\) and \(F_Y=\Phi_Y(\widetilde Y)\). \label{alg:transforms}

\If{\(\bm Z \) is nonempty} \label{alg:tree1}
    \State Select the \(X\)- and \(Y\)-side tree-size parameters separately by
out-of-bag prediction of \(G_X\) from \(\bm Z\) and \(F_Y\) from \(\bm Z\)
over \(\mathcal G\).
    \State Residualize \(G_X\) on \(\bm Z \) using trees, producing \(\widehat{\bm U}_X\). \label{alg:tree2}
    \State Residualize \(F_Y\) on \(\bm Z \) using trees, producing \(\widehat{\bm U}_Y\).
\Else
    \State Center \(G_X\) and \(F_Y\), producing \(\widehat{\bm U}_X\) and \(\widehat{\bm U}_Y\).
\EndIf \label{alg:tree3}

\State Compute the residual cross-covariance matrix:
\(
\widehat C
=
\frac1n \widehat{\bm U}_X^\top \widehat{\bm U}_Y.
\)

\State Compute the test statistic:
\(
S=n\|\widehat C\|_F^2.
\)

\State Estimate the null distribution from the eigenvalues of the product-feature covariance matrix.

\State Compute a moment-matched chi-square \(p\)-value.

\State \Return \(S\) and \(p\)-value.
\end{algorithmic}
\end{algorithm}

\section{Theoretical Results}
\subsection{Overall Strategy}

BLITZ uses the following residual characterization of conditional independence.

\begin{theorem}\label{thm:residuals}
Let \((\bm X,\bm Y,\bm Z)\) be random variables with \(\bm X\in\mathbb R^{d_X}\) and
\(\bm Y\in\mathbb R^{d_Y}\). Let \(m_{\bm X}\) and \(m_{\bm Y}\) be any measurable functions of
\(\bm Z \), and define
\[
R_{\bm X} \coloneqq \bm X-m_{\bm X}(\bm Z), \qquad R_{\bm Y} \coloneqq \bm Y-m_{\bm Y}(\bm Z).
\]
Then
\[
\bm X\perp\!\!\!\perp \bm Y\mid \bm Z
\quad\Longleftrightarrow\quad
\operatorname{Cov}\{f(R_{\bm X}),g(R_{\bm Y})\mid \bm Z\}=0
\quad\text{a.s.}
\]
for all square-integrable measurable functions \(f\) and \(g\).
\end{theorem}
\noindent We delegate proofs to Supplementary Materials \ref{supp:proofs}. Theorem~\ref{thm:residuals} separates the construction into two approximation
problems. First, we may choose \(m_{\bm X}\) and \(m_{\bm Y}\) from a computationally simple
function class to remove broad dependence on \(\bm Z \). Second, after forming
\(R_{\bm X}=\bm X-m_{\bm X}(\bm Z)\) and \(R_{\bm Y}=\bm Y-m_{\bm Y}(\bm Z)\), we apply transformations \(f\) and \(g\)
to the residuals and residualize \(f(R_{\bm X})\) and \(g(R_{\bm Y})\) with respect to
\(\bm Z \) using a potentially different learner. This two-stage construction lets
the initial residualization and the final conditional-association search cover
complementary parts of the function space. In practice, one need not enumerate
all square-integrable transformations. As in approximate kernel-based conditional
independence testing, a finite collection of suitably chosen features can
provide a tractable approximation to a richer function class \cite{Strobl19}.

The benefit of this construction is amplified by the form of the test
statistic. The statistic does not require either second-stage residualizer to
remove all remaining dependence on \(\bm Z \). Instead, the relevant bias term is a
product of the residual conditional-mean errors on the two sides.

\begin{lemma}[Product bound for residual covariance bias]
\label{lem:double}
Let \(\widehat{u}_X,\widehat{u}_Y\in\mathbb R^n\) denote the
second-stage BLITZ residuals for a chosen pair of transformed features, and let
\(\widehat C=n^{-1}\widehat{u}_X^\top\widehat{u}_Y\). Let
\(\mathcal Z_n=\sigma(\bm Z_1,\ldots,\bm Z_n)\), and define
\[
e_{X,n}
\coloneqq
\left(
\frac1n
\left\|
\mathbb E[\widehat{u}_X\mid \mathcal Z_n]
\right\|^2
\right)^{1/2},
\qquad
e_{Y,n}
\coloneqq
\left(
\frac1n
\left\|
\mathbb E[\widehat{u}_Y\mid \mathcal Z_n]
\right\|^2
\right)^{1/2}.
\]
Assume i.i.d. observations and \(H_0:X\perp\!\!\!\perp Y\mid \bm Z\). Then
\(\left|\mathbb E[\widehat C\mid \mathcal Z_n]\right|\le e_{X,n}e_{Y,n}\).
Consequently, if \(e_{X,n}e_{Y,n}\in o_p(n^{-1/2})\), then
\(\sqrt n\,\mathbb E[\widehat C\mid \mathcal Z_n]\to 0\) in probability. In
particular, this holds if \(e_{X,n}\in o_p(n^{-1/4})\) and
\(e_{Y,n}\in o_p(n^{-1/4})\).
\end{lemma}
\noindent Thus the two stages only need to make the product \(e_{X,n}e_{Y,n}\) small on
the covariance scale. For a root-\(n\) covariance statistic, the relevant
asymptotic requirement is \(e_{X,n}e_{Y,n}\in o_p(n^{-1/2})\), which in the
balanced case is achieved by only requiring each side to have conditional-mean error
\(o_p(n^{-1/4})\).

\subsection{Complementary Approximation by Two Function Classes}

Let \(\mathcal M_n\) denote the first-stage residualization class and let
\(\mathcal A_K\) denote a second-stage approximation class indexed by a
complexity parameter \(K\). Here \(K\) denotes the size of the second-stage
approximator, such as the number of tree leaves, basis functions, knots, or
random features. For a transformation \(\phi\), define the second-stage
regression target induced by a first-stage residualizer \(m_X\) as
\[
\eta_{\phi,m}(\bm z)
\coloneqq
\mathbb E\{\phi(X-m_X(\bm Z))\mid \bm Z=\bm z\}.
\]
The first-stage residualizer therefore does not merely reduce prediction error
for \(X\); it changes the function that the second-stage class must
approximate.

For a generic target \(\eta\), define
\(
\Delta_{\mathcal A_K}(\eta)
\coloneqq
\inf_{a\in\mathcal A_K}\|\eta-a\|_{L_2(\mathbb{P}_{\bm Z})}.
\) We make two assumptions:

\begin{assumption}[Second-stage approximation bound]\label{assump:standard_bound}
For the targets under consideration, the second-stage class satisfies a
standard approximation bound of the form
\begin{equation}\label{eq:standard_bound}
\Delta_{\mathcal A_K}(\eta)
\le
C\,\mathcal J(\eta)\,K^{-\alpha},
\end{equation}
for some \(C>0\) and \(\alpha>0\). Here \(C\) is independent of \(K\), while
\(\alpha\) gives the rate at which the approximation error decreases as \(K\)
grows. The \textit{second-stage complexity constant} \(\mathcal J(\eta)\)
measures the target-specific size or roughness within that rate class.
\end{assumption}

\begin{assumption}[Complementary complexity reduction]\label{assump:complexity_J}
For each side \(q\in\{X,Y\}\), there exists a first-stage residualizer
\(m_q^\star\in\mathcal M_n\) such that
\[
\mathcal J(\eta_{q,m_q^\star})
\le
\rho_q\,\mathcal J(\eta_{q,0}),
\qquad 0<\rho_q<1,
\]
where \(\eta_{q,0}\) denotes the target faced by the second-stage learner
without first-stage residualization.
\end{assumption}

Assumption~\ref{assump:complexity_J} does not require the first-stage class to
remove the second-stage target exactly. It only requires the first stage to
reduce the effective complexity of the target faced by the second stage. This
is reasonable when \(\mathcal M_n\) captures structure that is simple in a
global basis but inefficient for \(\mathcal A_K\) to approximate.

We obtain the following important result:
\begin{lemma}[Reduction in sufficient second-stage complexity]\label{lem:complexity_reduction}
Suppose Assumption~\ref{assump:standard_bound} holds. For a single side, write
\(\rho=\rho_q\). Under Assumption~\ref{assump:complexity_J}, the value of
\(K\) sufficient to approximate \(\eta_{q,m_q^\star}\) to tolerance
\(\tau>0\) is smaller than the value of \(K\) sufficient to approximate
\(\eta_{q,0}\) to the same tolerance by the factor \(\rho^{1/\alpha}\).
\end{lemma}
\noindent Lemma~\ref{lem:complexity_reduction} says that if the first-stage
residualizer lowers the second-stage complexity constant \(\mathcal J(\eta)\),
then the second-stage learner can reach the same approximation accuracy with a
smaller value of \(K\). The reduction in \(K\) is amplified when the
second-stage approximation rate is slow: since the factor is
\(\rho^{1/\alpha}\) and \(0<\rho<1\), smaller values of \(\alpha\) make the
reduction much larger. Thus even a modest reduction in \(\mathcal J(\eta)\) can
substantially reduce the second-stage complexity required to reach a fixed
tolerance.

This reduction in second-stage complexity is further amplified by the product structure of the covariance statistic:
\begin{proposition}[Product-complexity reduction for covariance bias]\label{prop:product_complexity}
Suppose the second-stage errors satisfy
\[
e_{X,n}\lesssim C_XJ_XK_X^{-\alpha},
\qquad
e_{Y,n}\lesssim C_YJ_YK_Y^{-\alpha}.
\]
If the first stage reduces the second-stage complexity constants by factors
\(\rho_X\) and \(\rho_Y\), then, at fixed \(K_X\) and \(K_Y\), the
covariance-bias bound is reduced by \(\rho_X\rho_Y\). Equivalently, to reach
the same covariance-bias tolerance, the sufficient product complexity
\(K_XK_Y\) is reduced by
\(
(\rho_X\rho_Y)^{1/\alpha}.
\)
In the symmetric case where \(\rho=\rho_X=\rho_Y\), this becomes
\(
\rho^{2/\alpha}.
\)
\end{proposition}
\noindent Thus the covariance statistic enhances the benefit of reducing the
second-stage complexity constant; a one-sided reduction factor
\(\rho^{1/\alpha}\) is applied on both sides, yielding a squared reduction in
the sufficient product complexity when the two sides have the same reduction
factor.

\subsection{Polynomial--Tree Pairing}

BLITZ instantiates the two-stage strategy specifically by using low-order polynomial
regression for the first-stage residualization and trees for the second-stage
residualization. This polynomial--tree pairing is especially attractive because
the two classes have complementary approximation strengths while remaining
computationally lightweight. Low-order polynomial regression is fast, stable,
and well suited to removing smooth global structure in the conditional mean,
including linear trends, quadratic curvature, and low-order interactions.
Trees, by contrast, approximate functions through recursive partitions and are
therefore naturally suited to localized, threshold-like, and step-wise
structure. Using trees alone to approximate a smooth global surface can require
many splits, while using low-order polynomials alone can miss localized
nonlinearities or interaction effects. The two-stage construction exploits the
strengths of both classes: the polynomial stage cheaply removes smooth
large-scale dependence on \(\bm Z \), and the tree stage spends its adaptivity on the
remaining residual structure.

The following result holds for trees approximating smooth H\"older functions:

\begin{lemma}[Piecewise-constant approximation of H\"older functions]\label{lem:trees}
Let \( s \geq 1\), let \(\mathcal Z\subset\mathbb R^s\) be a bounded rectangle, and let
\(h:\mathcal Z\to\mathbb R\) be \(\beta\)-H\"older, \(0<\beta\le 1\), with
seminorm
\[
[h]_{\beta,\mathcal Z}
\coloneqq
\sup_{\substack{z,z'\in\mathcal Z\\ z\neq z'}}
\frac{|h(z)-h(z')|}{\|z-z'\|^\beta}.
\]
Let \(\mathcal T_K\) contain all axis-aligned piecewise-constant functions with
at most \(K\) rectangular cells. Then, for all sufficiently large \(K\),
\[
\inf_{t\in\mathcal T_K}\|h-t\|_{L_2(\mathbb P_{\bm Z})}
\le
C_{\mathcal Z,s,\beta}[h]_{\beta,\mathcal Z}K^{-\beta/s},
\]
where \(C_{\mathcal Z,s,\beta}<\infty\) depends on the domain, dimension,
and smoothness exponent.
\end{lemma}
\noindent Thus, for \(K\)-leaf trees approximating \(\beta\)-H\"older functions on an
\(s\)-dimensional conditioning space, the approximation bound takes the form
\[
\Delta_{\mathcal T_K}(h)
\le
C [h]_{\beta}K^{-\beta/s},
\]
similar to Equation~\eqref{eq:standard_bound}, where
\(\alpha=\beta/s\) and \(\mathcal J(h)=[h]_{\beta}\). We have dropped the
subscripts of \(C\) and \([h]_{\beta}\) for ease of presentation.

The following theorem quantifies the leaf-budget reduction obtained by pairing polynomial preprocessing with tree-based second-stage residualization:
\begin{theorem}[Polynomial--tree leaf-budget reduction]\label{thm:poly_tree}
Let \(\eta\) be a second-stage target and let \(h=\eta-p\) be the target left
after polynomial preprocessing. Suppose
\[
[h]_{\beta}\le \rho[\eta]_{\beta},
\qquad 0<\rho<1.
\]
Then the sufficient one-sided tree leaf budget needed to reach the same
approximation tolerance is reduced by
\(
\rho^{s/\beta}.
\)
If the same reduction factor \(\rho\) holds on both sides of the covariance
statistic, then the
sufficient product leaf budget \(K_XK_Y\) is reduced by
\(
\rho^{2s/\beta}.
\)
\end{theorem}
\noindent Thus, with the covariance product structure, polynomial preprocessing helps in
two related ways. First, for each side separately, reducing the effective
H\"older seminorm by a factor \(\rho\) reduces the sufficient leaf budget by
\(\rho^{s/\beta}\). Second, because the covariance-bias bound is proportional
to \(e_{X,n}e_{Y,n}\), applying the same reduction on both sides reduces the
covariance-bias constant by \(\rho^2\); equivalently, the sufficient product
leaf budget \(K_XK_Y\) is reduced by \(\rho^{2s/\beta}\). For a root-\(n\)
covariance statistic, the relevant asymptotic requirement is
\(e_{X,n}e_{Y,n}\in o_p(n^{-1/2})\), which in the balanced case is achieved by
\(e_{X,n}\in o_p(n^{-1/4})\) and \(e_{Y,n}\in o_p(n^{-1/4})\).

In practice, these bounds can translate into \textit{large} reductions in the sufficient value of \(K\), often by orders of magnitude. A modest reduction in the second-stage complexity constant can produce a many-fold reduction in the sufficient tree size, and the effect is amplified in the covariance statistic because the two side errors enter multiplicatively. This makes it easier to control conditional mean bias under the null, so the covariance statistic can behave in practice as if the residuals were centered given \(\bm Z \). The following examples illustrate the scale of this reduction:

\textbf{Example 1: linear removal.}
Let \(Z\sim\mathrm{Unif}[-1,1]\) and
\(
\eta(z)=100z+\cos(\pi z).
\)
The population linear regression of \(\eta\) on \(\operatorname{span}\{1,z\}\)
removes the linear term, because \(\cos(\pi z)\) is orthogonal to both
\(1\) and \(\bm Z \) on \([-1,1]\). Thus the polynomial stage gives
\(p^\star(z)=100z\), and the tree stage sees
\[
h(z)=\eta(z)-p^\star(z)=\cos(\pi z).
\]
For \(\beta=1\), \([\eta]_1=\sup_{z\in[-1,1]}|100-\pi\sin(\pi z)|=100+\pi\) by the mean value theorem,
whereas \([h]_1=\sup_{z\in[-1,1]}\pi|\sin(\pi z)|=\pi\). Hence
\[
\rho=\frac{[h]_1}{[\eta]_1}=\frac{\pi}{100+\pi}.
\]
Since \(s=1\) and \(\beta=1\), the sufficient per-side leaf budget is reduced
by \(\rho^{s/\beta}=\pi/(100+\pi)\); equivalently, polynomial preprocessing reduces the sufficient per-side leaf budget by a factor of approximately \(33\). If the same reduction occurs on both
sides, the covariance-bias constant is reduced by
\[
\rho^2=\left(\frac{\pi}{100+\pi}\right)^2.
\]
In other words, the bound is smaller by a factor of approximately \(1078\).

\textbf{Example 2: quadratic removal.}
Let \(Z\sim\mathrm{Unif}[-1,1]\) and
\[
\eta(z)=100z^2+\left(z^3-\frac35z\right).
\]
The term \(z^3-\frac35z\) is orthogonal to \(\operatorname{span}\{1,z,z^2\}\)
on \([-1,1]\). Therefore the population quadratic regression removes the
quadratic component exactly, giving \(p^\star(z)=100z^2\), and the tree stage
sees
\[
h(z)=z^3-\frac35z.
\]
For \(\beta=1\),
\[
[\eta]_1
=
\sup_{z\in[-1,1]}\left|200z+3z^2-\frac35\right|
=
\frac{1012}{5},
\]
while
\[
[h]_1
=
\sup_{z\in[-1,1]}\left|3z^2-\frac35\right|
=
\frac{12}{5}.
\]
Hence
\[
\rho=\frac{[h]_1}{[\eta]_1}
=
\frac{12}{1012}
=
\frac{3}{253}.
\]
Again \(s=1\) and \(\beta=1\), so the sufficient per-side leaf budget is
reduced by \(3/253\), or a factor of approximately 84. If both sides obtain this reduction, the covariance-bias
constant is reduced by
\(
\rho^2=\left(\frac{3}{253}\right)^2,
\) so that the bound is reduced by a factor of approximately \(7112\).

\vspace{3mm}
We conclude that first-stage polynomial residualization can substantially reduce the tree size needed to control residual cross-covariance bias. This supports practical Type I error control while preserving power, because the second-stage trees can remain small enough to avoid excessive overfitting. 

\subsection{Asymptotic Distribution}

We now derive the asymptotic distribution of the test statistic \(S\). Let
\(\widehat{\bm U}_{X,i}\in\mathbb R^p\) and
\(\widehat{\bm U}_{Y,i}\in\mathbb R^q\) denote the final estimated residualized
feature vectors produced by Algorithm~\ref{alg:blitz}, and define
\(
\widehat{\bm W}_i
\coloneqq
\operatorname{vec}\!\left(
\widehat{\bm U}_{X,i}\widehat{\bm U}_{Y,i}^{\top}
\right)
\in\mathbb R^{pq}.
\)
Also define
\[
e_{X,n}
=
\left\{
\frac1n\sum_{i=1}^{n}
\left\|
\mathbb E(\widehat{\bm U}_{X,i}\mid \mathcal Z_n)
\right\|^2
\right\}^{1/2},
\qquad
e_{Y,n}
=
\left\{
\frac1n\sum_{i=1}^{n}
\left\|
\mathbb E(\widehat{\bm U}_{Y,i}\mid \mathcal Z_n)
\right\|^2
\right\}^{1/2}.
\]
These quantities are the vector-valued analogues of the scalar residual
conditional-mean errors in Lemma~\ref{lem:double}. The following result holds:

\begin{theorem}[Weighted chi-square null law for the BLITZ statistic]
\label{thm:weighted_chisq}
Assume i.i.d. observations and \(H_0:X\perp\!\!\!\perp Y\mid \bm Z\). 
Assume that the vector-valued residual conditional-mean errors satisfy
\(e_{X,n}e_{Y,n}\in o_p(n^{-1/2})\).
Further suppose that the centered estimated product features satisfy
\[
\frac{1}{\sqrt n}\sum_{i=1}^{n}
\left[
\widehat{\bm W}_i-\mathbb{E}\!\left(\widehat{\bm W}_i\mid \mathcal Z_n\right)
\right]
\xrightarrow{d}
N(0,\Omega),
\]
where \(\Omega\in\mathbb R^{pq\times pq}\) is finite and positive
semidefinite. Then
\[
S
=
n\|\widehat C\|_F^2
=
\left\|
\frac{1}{\sqrt n}\sum_{i=1}^{n}\widehat{\bm W}_i
\right\|^2
\xrightarrow{d}
\sum_{r=1}^{pq}\lambda_r\chi_{1,r}^{2},
\]
where \(\lambda_1,\ldots,\lambda_{pq}\) are the eigenvalues of \(\Omega\), and
\(\chi_{1,1}^{2},\ldots,\chi_{1,pq}^{2}\) are independent chi-square random
variables with one degree of freedom.
\end{theorem}

\hspace{4.5mm} We interpret Theorem~\ref{thm:weighted_chisq} as a high-level null
approximation result. Rather than requiring either residualizer to be perfectly
consistent, it requires only that the product of the two residual
conditional-mean errors be \(o_p(n^{-1/2})\). This condition is plausible in
the low- and moderate-dimensional regimes targeted by BLITZ because the
polynomial stage removes broad smooth dependence on \(\bm Z\), while the
shallow tree stage approximates the lower-complexity residual conditional
means. For H\"older-smooth second-stage targets, the tree approximation error
scales as \(K^{-\beta/s}\), and polynomial preprocessing can drastically reduce the
effective H\"older seminorm faced by the tree residualizer. Thus, in regimes
where \(s=\dim(\bm Z)\) is modest and the remaining conditional-mean structure
is smooth or locally threshold-like, the residual conditional-mean product can
be small enough for the centered product-feature CLT to dominate the statistic.
The result is not intended as an assumption-free guarantee; in high-dimensional,
highly irregular, or severely misspecified settings, the residualization
condition may fail.

Under the stated conditions, the limiting null distribution in
Theorem~\ref{thm:weighted_chisq} is a weighted sum of independent chi-square
variables. In finite samples, BLITZ estimates this distribution by replacing
the population eigenvalues \(\lambda_r\) with the empirical eigenvalues obtained
from the product-feature covariance matrix. The resulting weighted chi-square
law is then approximated using a moment-matched distribution. By default, we
use the four-moment Lindsay--Pilla--Basak approximation
\cite{Lindsay00}, and if this approximation fails numerically, we fall back to
the Hall--Buckley--Eagleson approximation \cite{Hall83,Buckley88}.

\section{Empirical Results}

We now test whether the theoretical finite-sample gains appear empirically and whether they translate into improved causal discovery performance.

\subsection{Hyperparameters \& Comparators}

We used the following default BLITZ hyperparameters throughout. The first-stage residualizer was degree-two polynomial regression including all interactions. The second-stage residualizer fit one shallow regression tree for each transformed feature on each side of the test. The minimum node size was selected separately for the \(X\)- and \(Y\)-sides by out-of-bag prediction over the grid \(\{5,10,15,20,40,80\}\).

We compared BLITZ with these default hyperparameters against the following fast CI testing methods, each run with its default hyperparameters:

\begin{enumerate}
\item \textbf{RCIT} \cite{Strobl19}: a random-Fourier-feature approximation to KCIT that replaces expensive kernel operations with finite-dimensional randomized features.
\item \textbf{RCoT} \cite{Strobl19}: a closely related random-feature relaxation of KCIT based on a randomized conditional-correlation statistic, typically recommended over RCIT due to improved Type I error rate control and power.
\item \textbf{FCIT} \cite{Chalupka18}: a prediction-based CI test that rejects the null when adding the candidate variable to the conditioning set significantly improves out-of-sample prediction.
\item \textbf{CCI} \cite{Ramsey14}: a regression-based CI test that first removes the effect of $\bm Z$ from $X$ and $Y$, then checks whether the resulting residuals still exhibit nonlinear dependence by testing a truncated collection of basis-function correlations $\operatorname{cov}(f(R_X), g(R_Y))$.
\item \textbf{FastKCI} \cite{Schacht25}: a partitioned kernel CI test that clusters the conditioning space with a mixture-of-experts model, runs local kernel CI tests in parallel on the blocks, and then aggregates them with importance weights.
\end{enumerate}
Finally, we performed \textbf{ablation analyses} by comparing BLITZ with two variants that remove either the first-stage polynomial residualizer or the second-stage tree residualizer.

\subsection{Type I Error}

We first assessed Type I error calibration and runtime under a post-nonlinear null model across sample sizes
\(n \in \{500, 1000, 2000, 5000, 10000\}\)
and conditioning dimensions
\(s \in \{0,1,2,3,4,5\}\).
For each simulation, we generated data in which \(X\) and \(Y\) were conditionally independent given shared Gaussian covariates \(\bm Z \). Specifically, we sampled \(Z \in \mathbb{R}^{n \times s}\) with independent standard Gaussian entries and generated
\[
X = f_X(Zw_X + \omega_X E_X), \qquad
Y = f_Y(Zw_Y + \omega_Y E_Y),
\]
where \(E_X,E_Y \in \mathbb{R}^n\) are independent standard Gaussian noise vectors,
\(w_X,w_Y \sim \mathrm{Unif}(0.5,1.5)^s\) are independently sampled covariate-weight vectors, and
\(\omega_X,\omega_Y \sim \mathrm{Unif}(0,1)\) are independently sampled noise weights.
We independently selected \(f_X\) and \(f_Y\) from five transformations: the identity map, square, cube, radial exponential decay \(\exp(-|x|)\), and hyperbolic tangent \(\tanh(x)\).
For each generated dataset, we applied every method and recorded its \(p\)-value and elapsed runtime.
We repeated this procedure 1000 times for each combination of sample size and conditioning dimension. Because constraint-based causal discovery algorithms require many repeated CI queries, we imposed a three-minute wall-clock limit on each run of each CI test.

\begin{figure}[!b]
    \centering
    \includegraphics[width=0.875\linewidth]{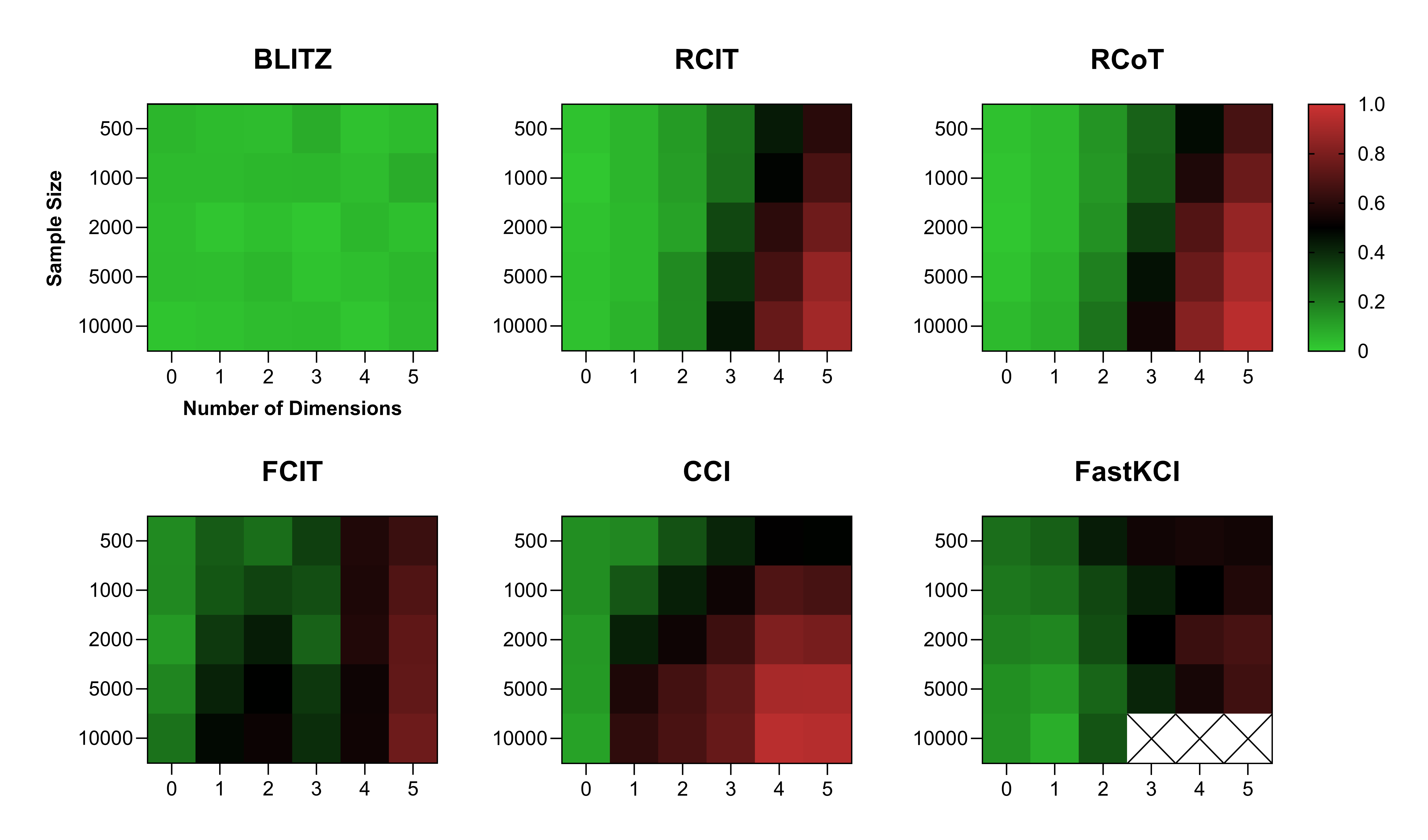}
    \caption{
    \textbf{Null calibration.}
    All heatmaps share the same x- and y-axis labels. Type I error calibration was measured by the Kolmogorov--Smirnov distance between each method's empirical \(p\)-value distribution and the \(\mathrm{Unif}(0,1)\) reference distribution. Lower values (greener) indicate better calibration. BLITZ achieved the smallest deviations from the uniform null distribution across the evaluated sample sizes and conditioning dimensions.
    }
    \label{fig:Type1:KS}
\end{figure}

\begin{figure}
    \centering
    \includegraphics[width=1\linewidth]{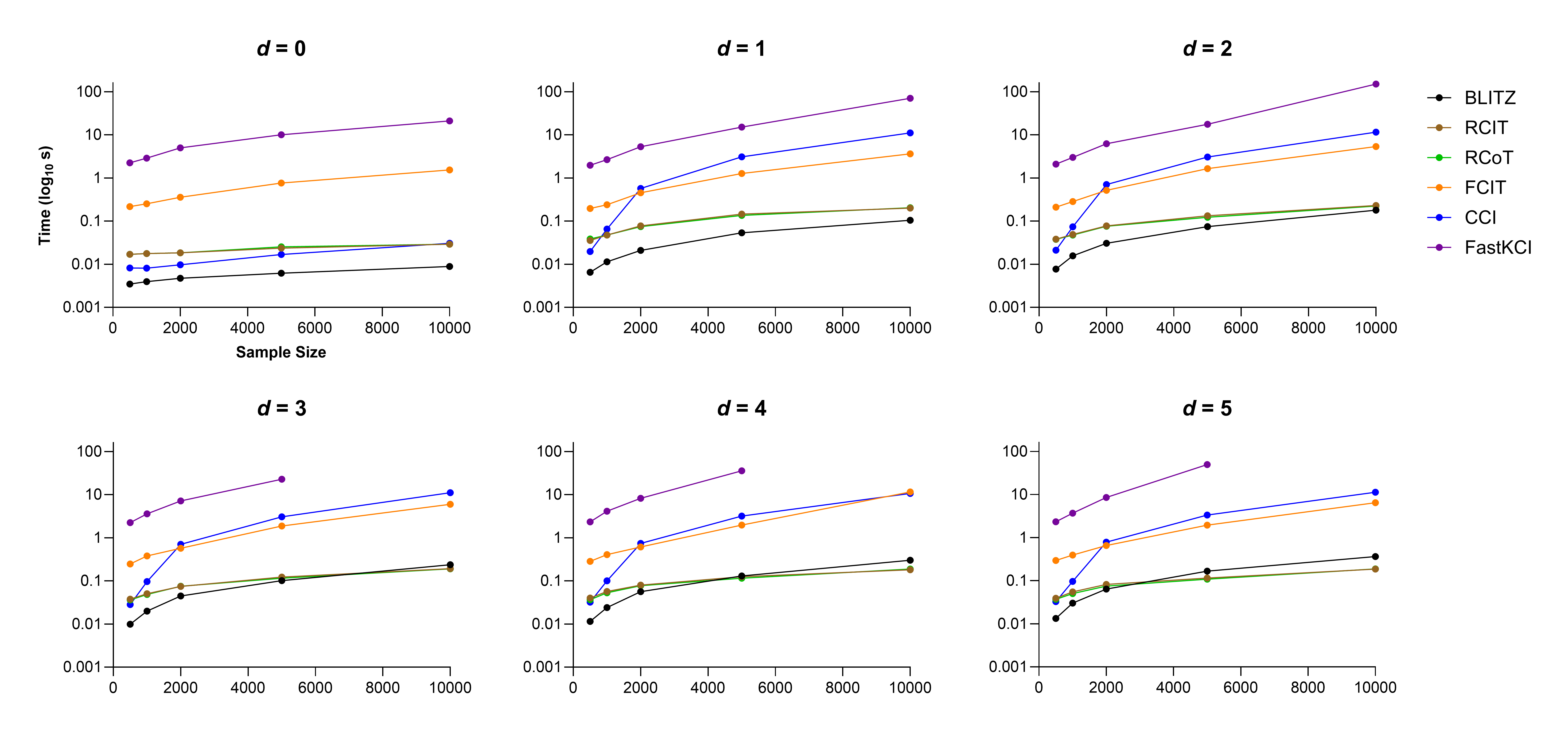}
    \caption{
    \textbf{Runtime under the null.}
    Mean wall-clock runtime for each CI test. BLITZ was the fastest method for \(s \in \{0,1,2\}\) and remained highly competitive at larger conditioning dimensions. Although RCIT and RCoT were faster in some higher-dimensional settings, their largest runtime advantage over BLITZ was less than 0.1 seconds. FastKCI was substantially slower and did not complete within the three-minute limit for \(n=10000\) and \(s \in \{3,4,5\}\).
    }
    \label{fig:Type1:time}
\end{figure}

We assessed calibration by comparing the empirical null distribution of each method's \(p\)-values with the \(\mathrm{Unif}(0,1)\) distribution using the \textbf{Kolmogorov--Smirnov statistic}.
Figure~\ref{fig:Type1:KS} summarizes these results.
Across the evaluated regimes, BLITZ produced the best-calibrated null \(p\)-values, achieving the smallest deviation from the uniform reference distribution.
RCIT and RCoT were the next strongest methods overall, although their calibration deteriorated at larger sample sizes and higher conditioning dimensions, where the residualization step must capture increasingly complex conditional structure.
FCIT also showed degraded calibration, with the deterioration driven primarily by increases in the conditioning dimension.
FastKCI exhibited a similar pattern to FCIT and, in addition, failed to complete within the three-minute limit when \(n=10000\) and \(s \in \{3,4,5\}\).
CCI performed poorly in this evaluation because its reported scores did not behave as calibrated \(p\)-values under the null. Ablation analyses revealed that removing the polynomial and tree steps led to anticonservativeness (Supplementary Figure \ref{fig:Type1:ablations}). In particular, the no-polynomial ablation was sometimes close to BLITZ under the ordinary KS statistic, but the positive KS statistic showed excess small \(p\)-values, indicating residual anti-conservativeness rather than harmless random fluctuation. We conclude that BLITZ maintained the best control of the Type I error rate.

We summarize runtime performance in Figure~\ref{fig:Type1:time}. The methods differed substantially in computational cost. BLITZ was the fastest method at lower conditioning dimensions, completing within 0.365 seconds across all conditions. At higher dimensions, RCIT and RCoT became competitive, although their maximum runtime advantage over BLITZ was only 0.099 seconds at \(n=10000\) and \(s=5\). In contrast, FastKCI was by far the slowest method, with BLITZ running approximately 122 to 2357 times faster across the evaluated regimes. These results indicate that BLITZ is not only well calibrated under the null but also among the fastest CI tests considered.

\subsection{Power}

We next assessed statistical power under alternatives to conditional independence. We used the same sample sizes, conditioning dimensions, nonlinear transformations, noise structure, and 1000 simulation replicates as in the Type I error experiment. For \(\dim(\bm Z)>0\), we generated dependence by first using unperturbed covariates to generate \(X\) and \(Y\), and then adding small Gaussian perturbations \(\mathcal{N}(0,1/16)\) to the observed conditioning variables supplied to the CI tests. Thus, \(X\) and \(Y\) shared information through the latent, unperturbed covariates, while the CI tests conditioned only on the perturbed version of \(\bm Z\). For the marginal case \(\dim(\bm Z)=0\), we instead generated \(X\) and \(Y\) from a shared unobserved one-dimensional Gaussian variable plus independent noise, and then supplied no conditioning variables to the CI tests. In practice, these constructions create moderately difficult alternatives: the signal is strong enough to detect, but not so strong that all methods saturate across the tested sample sizes and conditioning dimensions.

We evaluated power using the \textbf{area under the calibrated power curve (AUCPC)}. For each method, sample size, and conditioning dimension, we first used the null simulations to determine the \(p\)-value cutoff that achieved a target empirical Type I error rate \(\alpha\). We then applied this calibrated cutoff to the alternative simulations to estimate power at that \(\alpha\). Repeating this procedure over a grid of \(\alpha\) values and integrating the resulting calibrated power curve yielded AUCPC. This metric summarizes power across rejection thresholds while controlling for differences in null calibration, preventing anti-conservative methods from appearing artificially powerful simply because they produce overly small \(p\)-values under the null.

\begin{figure}[!t]
    \centering
    \includegraphics[width=0.875\linewidth]{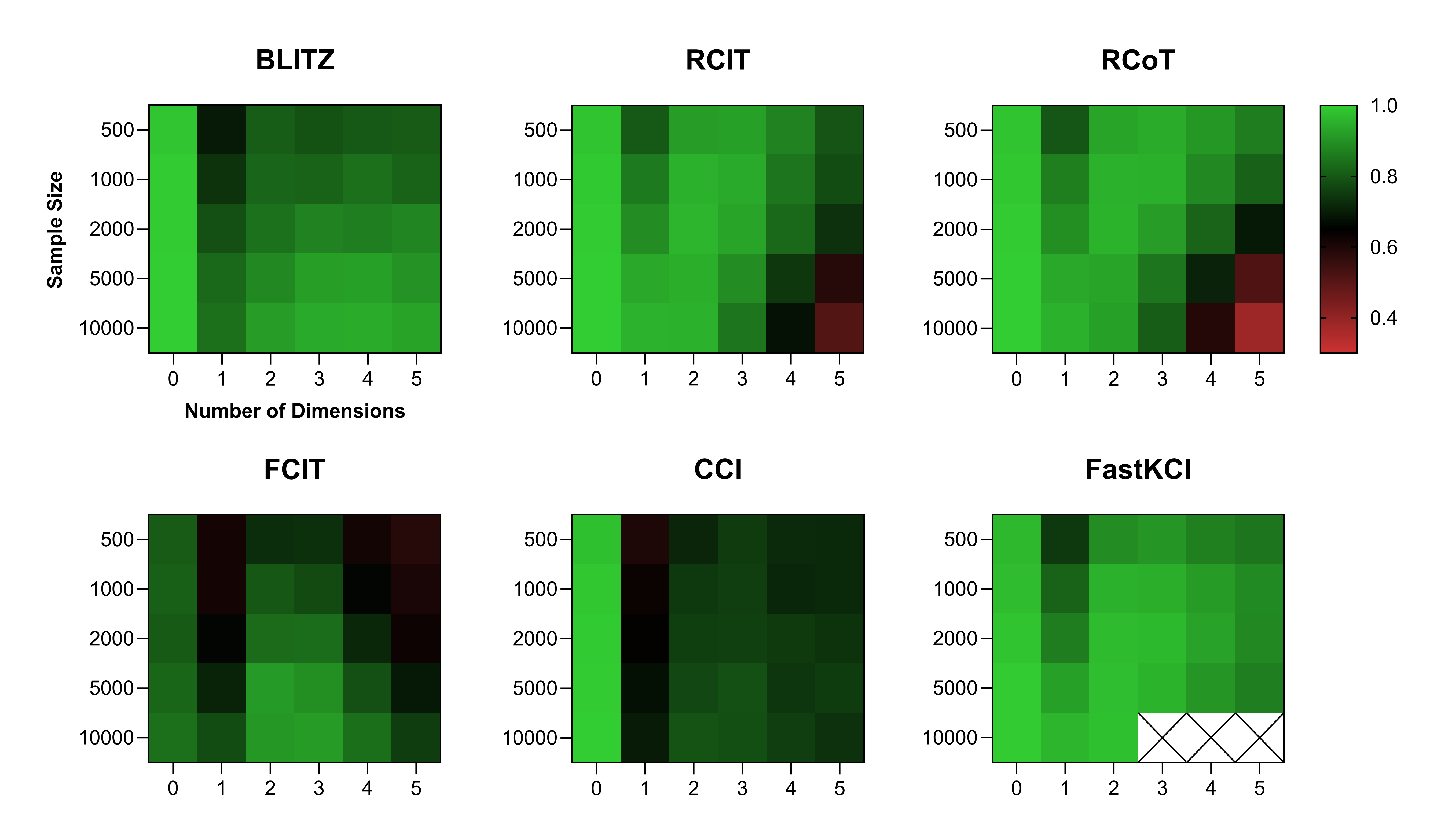}
    \caption{
    \textbf{Power across standalone CI benchmarks.} Heatmaps report AUCPC values across sample sizes and conditioning dimensions, with larger values (greener) indicating greater size-calibrated power. BLITZ retained competitive power across the full experimental grid, including the more difficult conditional settings, whereas several competing methods showed less stable gains with increasing sample size or did not complete in the largest regimes.}

    \label{fig:TypeTwo}
\end{figure}

Figure~\ref{fig:TypeTwo} summarizes the AUCPC results. All methods performed well in the non-conditional setting, where the task reduces to detecting marginal dependence. The conditional settings were substantially more challenging, but BLITZ achieved power close to that of the strongest competing methods. The ablation results further showed that the two-stage residualization strategy preserved power, with only modest loss relative to the strongest no-polynomial variant (Supplementary Figure~\ref{fig:Type2:ablations}).
BLITZ's AUCPC also followed the expected qualitative pattern of increasing with sample size, whereas RCIT, RCoT, and CCI showed less consistent gains across the grid. FastKCI performed well, but did not scale beyond \(n=10000\) and \(s=2\). Timing results again confirmed that BLITZ completed consistently well under one second (Supplementary Figure~\ref{fig:Type2:time}). Taken together with the Type I error results, these findings suggest that BLITZ offers the strongest overall balance between calibration, power, and speed; it maintains valid operating characteristics across the full experimental grid, retains competitive sensitivity to dependence under the alternative, and consistently completes well under one second.

\subsection{Causal Discovery}

Although BLITZ controlled the Type I error rate well and achieved competitive power in standalone CI testing, the key practical question is whether these improvements translate into better causal graph recovery. We therefore evaluated each CI test within three constraint-based causal discovery algorithms: PC, FCI, and RFCI. To do this, we first generated random DAGs with 15 total variables and expected observed degree 2. For the fully observed PC setting, all variables were observed. For the FCI and RFCI settings, each replicate included either zero or one latent variable and either zero or one selection variable. Latent variables were generated as unobserved root causes of at least two observed variables, while selection variables were generated as unobserved sinks with at least two observed parents.

We obtained 5000 samples from each graph in topological order. For each node, we formed a weighted sum of its parents plus Gaussian noise, with edge weights sampled uniformly in magnitude between 0.25 and 1 and assigned random signs. We then applied a randomly selected post-nonlinear transformation to each node: identity, square, cube, radial exponential decay, or hyperbolic tangent, and standardized the resulting variable. For selection-variable settings, we generated an oversized raw sample and retained observations with high selection-variable scores until the target sample size was reached. The algorithms were given only the observed variables. We repeated the above procedure 50 times.

Our primary endpoint-level metric was the \textbf{conditional endpoint F1 score}, or the F1 score for endpoint recovery conditional on the skeleton discovered by the algorithm. This metric evaluates orientation quality on the learned skeleton rather than over all possible node pairs: a well-calibrated CI test may remove more edges than an anti-conservative test, leading to lower adjacency recall, but the adjacencies that remain should be supported by more reliable separating sets and therefore yield more accurate endpoint orientations. This also reflects how applied scientists typically interpret causal discovery output: after inspecting the recovered graph, they reason about causal directionality among the adjacencies that the algorithm retained. We also report the corresponding \textbf{conditional endpoint accuracy} and \textbf{conditional endpoint MCC} (Matthew's correlation coefficient).

As a secondary diagnostic, we report the \textbf{minimum-oracle separating-set size match rate}. For each true nonedge, we compute the minimum separating-set size implied by the oracle graph, and then measure the proportion of discovered separating sets whose cardinality matches this oracle minimum. This evaluates whether each CI test recovers separating sets at the correct conditioning order, rather than merely whether it removes an edge. Finally, we report \textbf{runtime}. Other structural recovery metrics, including exact-endpoint F1/accuracy/MCC, and adjacency F1/accuracy/MCC, are reported in the Supplementary Materials.

We were unable to include FastKCI in the experiments because a single run required more than five hours, even when used with the PC algorithm and four cores. We also excluded CCI under FCI because individual runs exceeded seven hours. Table~\ref{table:causal_discovery} reports results for all remaining feasible combinations of CI tests and causal discovery algorithms. The first three row blocks correspond to PC, FCI, and RFCI, respectively, while the final block reports the mean within-algorithm rank across the three algorithms.

\begin{sidewaystable}
\centering
\begin{threeparttable}
\begin{tabular}{lcccccc}
\hline
Algorithm & Test & F1 & Accuracy & MCC & Sep-Size Match & Time (s) \\ \midrule

\multicolumn{1}{l}{\multirow{7}{*}{\hspace{-2mm}PC}} 
& BLITZ & 0.790 [0.742, 0.837] & 0.786 [0.738, 0.833] & 0.525 [0.420, 0.629] & 0.857 [0.829, 0.884] & 15.46 [11.41, 19.51] \\
\multicolumn{1}{c}{} 
& RCIT & 0.782 [0.736, 0.827] & 0.781 [0.735, 0.827] & 0.510 [0.404, 0.616] & 0.845 [0.816, 0.875] & 28.54 [21.65, 35.43] \\
\multicolumn{1}{c}{} 
& -poly & 0.780 [0.738, 0.823] & 0.778 [0.735, 0.821] & 0.520 [0.421, 0.618] & 0.857 [0.830, 0.885] & 15.36 [11.60, 19.13] \\
\multicolumn{1}{c}{} 
& RCoT & 0.770 [0.728, 0.812] & 0.769 [0.727, 0.810] & 0.478 [0.378, 0.579] & 0.840 [0.810, 0.870] & 28.10 [21.48, 34.73] \\
\multicolumn{1}{c}{} 
& CCI & 0.736 [0.697, 0.775] & 0.707 [0.666, 0.748] & 0.430 [0.352, 0.508] & 0.810 [0.782, 0.838] & 835.2 [595.2, 1075.] \\
\multicolumn{1}{c}{} 
& -tree & 0.735 [0.686, 0.784] & 0.698 [0.646, 0.749] & 0.444 [0.360, 0.529] & 0.839 [0.812, 0.867] & 3.730 [3.100, 4.350] \\
\multicolumn{1}{c}{} 
& FCIT & 0.729 [0.681, 0.777] & 0.717 [0.668, 0.766] & 0.392 [0.275, 0.510] & 0.758 [0.722, 0.794] & 148.9 [133.2, 164.5] \\ \hline

\multirow{6}{*}{FCI} 
& RCIT & 0.731 [0.676, 0.785] & 0.731 [0.676, 0.785] & 0.581 [0.502, 0.660] & 0.682 [0.614, 0.751] & 659.5 [280.9, 1038.] \\
& FCIT & 0.704 [0.640, 0.769] & 0.697 [0.633, 0.761] & 0.496 [0.362, 0.631] & 0.593 [0.524, 0.662] & 235.5 [154.9, 316.0] \\
& BLITZ & 0.703 [0.645, 0.760] & 0.701 [0.644, 0.759] & 0.550 [0.475, 0.625] & 0.702 [0.636, 0.768] & 709.3 [167.1, 1252.] \\
& RCoT & 0.684 [0.628, 0.740] & 0.684 [0.628, 0.740] & 0.502 [0.430, 0.575] & 0.686 [0.619, 0.753] & 784.8 [442.3, 1127.] \\
& -poly & 0.672 [0.610, 0.733] & 0.670 [0.609, 0.731] & 0.474 [0.391, 0.558] & 0.702 [0.635, 0.769] & 707.1 [254.0, 1160.] \\
& -tree & 0.583 [0.524, 0.641] & 0.564 [0.507, 0.621] & 0.367 [0.302, 0.433] & 0.690 [0.624, 0.756] & 220.1 [102.0, 338.2] \\ \hline

\multirow{7}{*}{RFCI} 
& FCIT & 0.698 [0.637, 0.759] & 0.693 [0.631, 0.755] & 0.437 [0.333, 0.540] & 0.592 [0.522, 0.661] & 144.3 [124.5, 164.0] \\
& BLITZ & 0.620 [0.559, 0.681] & 0.617 [0.556, 0.677] & 0.399 [0.321, 0.478] & 0.701 [0.635, 0.767] & 21.03 [15.59, 26.47] \\
& -poly & 0.607 [0.545, 0.668] & 0.603 [0.542, 0.665] & 0.413 [0.343, 0.483] & 0.702 [0.635, 0.769] & 20.83 [15.36, 26.30] \\
& RCoT & 0.575 [0.510, 0.640] & 0.574 [0.509, 0.638] & 0.377 [0.291, 0.462] & 0.689 [0.622, 0.756] & 51.59 [35.70, 67.48] \\
& RCIT & 0.573 [0.511, 0.634] & 0.572 [0.510, 0.633] & 0.377 [0.291, 0.463] & 0.687 [0.621, 0.753] & 47.68 [33.55, 61.81] \\
& -tree & 0.536 [0.477, 0.594] & 0.511 [0.454, 0.567] & 0.320 [0.259, 0.381] & 0.689 [0.623, 0.755] & 5.150 [3.970, 6.330] \\
& CCI & 0.535 [0.474, 0.596] & 0.523 [0.463, 0.584] & 0.323 [0.245, 0.400] & 0.686 [0.622, 0.750] & 1367. [824.1, 1910.] \\ \hline

\multirow{7}{*}{Avg Rank} 
& BLITZ & \textbf{2.00} & \textbf{1.67} & \textbf{2.00} & \underline{1.67} & 3.67 \\
& RCIT & \underline{2.67} & \underline{2.67} & \underline{2.83} & 4.33 & 4.00 \\
& FCIT & 3.33 & 3.00 & 4.00 & 6.67 & 4.67 \\
& -poly & 3.67 & 3.67 & 3.00 & \textbf{1.33} & \underline{2.67} \\
& RCoT & 4.00 & 4.00 & 3.83 & 3.83 & 5.00 \\
& -tree & 6.00 & 7.00 & 6.00 & 3.83 & \textbf{1.00} \\
& CCI & 6.00 & 6.00 & 6.00 & 6.50 & 7.00 \\ \hline

\end{tabular}
\caption{\textbf{Causal discovery results.} The first three row blocks correspond
to PC, FCI, and RFCI, while the final block reports the mean within-algorithm
rank. Entries are means across simulation replicates with 95\% confidence
intervals. Lower ranks indicate better performance. Best values are bolded and
second-best values are underlined. The ``-poly'' and ``-tree'' ablations remove
the polynomial and tree-based residualization stages of BLITZ, respectively.}
\label{table:causal_discovery}
\end{threeparttable}
\end{sidewaystable}

Overall, BLITZ achieves the best average rank for conditional endpoint F1, accuracy, and MCC, indicating that it provides the most reliable endpoint orientations among the adjacencies retained by the discovery algorithm. Decomposing F1 into its components, BLITZ further achieves the best precision and recall (Supplementary Table~\ref{table:conditional_endpoint_precision_recall}). Finally, excluding ablations, BLITZ was the fastest CI test overall.

Consistent with our proposed mechanism, BLITZ achieves strong minimum-oracle separating-set size match rates, suggesting that it more often recovers separating sets at the correct conditioning order for true nonedges. In contrast, RCIT, RCoT, and CCI more frequently deviated from the minimum oracle separating-set size, consistent with the hypothesis that anti-conservative CI tests distort the separating sets used for orientation. BLITZ’s unconditional exact endpoint recovery metrics are more moderate, with the exception of precision and F1, and its adjacency accuracy is near average (Supplementary Table~\ref{table:additional_structural_metrics}). Together, these results suggest that BLITZ’s primary advantage lies in endpoint orientation conditional on retained adjacencies, rather than in uniformly improving all aspects of graph recovery.

\subsection{Real Data}

We next evaluated the CI tests within each causal discovery algorithm on the CYTO dataset, which contains single-cell flow cytometry measurements of 11 phosphoproteins and phospholipids \cite{Sachs05}. We used a preprocessed version of the dataset from \cite{Ramsey18}. To remove intervention-specific shifts induced by the corresponding root nodes, we standardized each variable to have mean zero and unit variance within each intervention condition. Each algorithm was then run over 50 bootstrap resamples.

Table \ref{table:cyto_causal_discovery} summarizes the CYTO causal discovery results. We do not report MCC because it was undefined for most algorithms, reflecting degenerate endpoint label distributions in which the MCC denominator collapsed. BLITZ achieved the highest overall F1 score and the second-highest accuracy for conditional endpoint discovery. The precision--recall decomposition showed that BLITZ also attained the highest recall and the second highest precision (Supplementary Table \ref{table:cyto_conditional_endpoint_precision_recall}). Although FCIT achieved the highest conditional endpoint accuracy, it also had very low adjacency recall, indicating that it recovered a sparse graph and evaluated endpoint marks on only a small subset of retained edges. BLITZ was also the fastest non-ablation method, completing the analysis approximately 2--3 times faster than RCIT and RCoT. The separating-set match rate was approximately 50\% across algorithms. Overall, the CYTO results are consistent with the synthetic experiments: BLITZ was the fastest non-ablation method and remained among the strongest methods across the main evaluation criteria, ranking first or second on the principal endpoint recovery metrics when all methods were considered.

\begin{table}
\centering
\begin{tabular}{lccccc}
\hline
Algorithm & Test & F1 & Accuracy & SepSizeMatch & Time \\ \midrule
\multirow{7}{*}{PC} & -poly & 0.841 [0.840, 0.842] & 0.799 [0.798, 0.800] & 0.499 [0.490, 0.509] & 5.180 [4.827, 5.534] \\
 & BLITZ & 0.840 [0.838, 0.843] & 0.798 [0.795, 0.801] & 0.499 [0.490, 0.509] & 4.905 [4.566, 5.244] \\
 & FCIT & 0.838 [0.820, 0.857] & 0.811 [0.787, 0.835] & 0.531 [0.521, 0.541] & 140.2 [130.3, 150.0] \\
 & RCIT & 0.798 [0.772, 0.824] & 0.754 [0.728, 0.780] & 0.495 [0.481, 0.509] & 14.52 [13.54, 15.51] \\
 & CCI & 0.774 [0.744, 0.805] & 0.728 [0.696, 0.761] & 0.491 [0.476, 0.505] & 388.8 [340.1, 437.6] \\
 & RCoT & 0.734 [0.701, 0.768] & 0.682 [0.647, 0.716] & 0.491 [0.480, 0.503] & 15.79 [14.75, 16.84] \\
 & -tree & 0.609 [0.574, 0.644] & 0.552 [0.517, 0.588] & 0.499 [0.489, 0.509] & 1.333 [1.263, 1.403] \\
\hline
\multirow{7}{*}{FCI} & FCIT & 0.850 [0.838, 0.861] & 0.825 [0.808, 0.843] & 0.538 [0.528, 0.548] & 137.0 [126.7, 147.4] \\
 & BLITZ & 0.841 [0.839, 0.843] & 0.799 [0.796, 0.801] & 0.499 [0.490, 0.509] & 4.844 [4.518, 5.170] \\
 & -poly & 0.841 [0.839, 0.843] & 0.799 [0.796, 0.801] & 0.499 [0.490, 0.509] & 5.343 [4.958, 5.728] \\
 & RCoT & 0.840 [0.834, 0.846] & 0.797 [0.790, 0.803] & 0.493 [0.481, 0.505] & 34.09 [21.08, 47.11] \\
 & RCIT & 0.832 [0.820, 0.844] & 0.789 [0.776, 0.802] & 0.495 [0.482, 0.507] & 19.00 [15.11, 22.89] \\
 & CCI & 0.823 [0.801, 0.844] & 0.780 [0.759, 0.802] & 0.492 [0.477, 0.507] & 646.0 [490.4, 801.6] \\
 & -tree & 0.634 [0.595, 0.672] & 0.580 [0.541, 0.619] & 0.501 [0.490, 0.511] & 10.64 [4.365, 16.91] \\
\hline
\multirow{7}{*}{RFCI} & BLITZ & 0.841 [0.840, 0.842] & 0.799 [0.798, 0.800] & 0.499 [0.490, 0.509] & 4.884 [4.549, 5.219] \\
 & -poly & 0.841 [0.839, 0.843] & 0.799 [0.796, 0.801] & 0.499 [0.490, 0.509] & 5.136 [4.769, 5.504] \\
 & FCIT & 0.821 [0.800, 0.841] & 0.794 [0.768, 0.821] & 0.529 [0.520, 0.538] & 131.4 [122.3, 140.6] \\
 & RCIT & 0.812 [0.791, 0.833] & 0.767 [0.744, 0.790] & 0.501 [0.489, 0.512] & 14.47 [13.47, 15.46] \\
 & CCI & 0.762 [0.724, 0.799] & 0.718 [0.680, 0.756] & 0.491 [0.477, 0.506] & 394.6 [346.1, 443.1] \\
 & RCoT & 0.731 [0.690, 0.773] & 0.681 [0.639, 0.722] & 0.497 [0.485, 0.509] & 15.72 [14.50, 16.93] \\
 & -tree & 0.578 [0.537, 0.619] & 0.525 [0.484, 0.565] & 0.499 [0.489, 0.509] & 1.324 [1.250, 1.398] \\
\hline
\multirow{7}{*}{Avg Rank} & BLITZ & \textbf{1.83} & \underline{2.17} & 3.50 & \textbf{1.67} \\
 & -poly & \textbf{1.83} & \underline{2.17} & 3.50 & \underline{2.67} \\
 & FCIT & \underline{2.33} & \textbf{1.67} & \textbf{1.00} & 6.00 \\
 & RCIT & 4.33 & 4.33 & 4.00 & 4.00 \\
 & CCI & 5.33 & 5.33 & 7.00 & 7.00 \\
 & RCoT & 5.33 & 5.33 & 6.00 & 5.00 \\
 & -tree & 7.00 & 7.00 & \underline{3.00} & \textbf{1.67} \\
\hline
\end{tabular}
\caption{\textbf{CYTO causal discovery results.} BLITZ showed the same qualitative pattern observed in the synthetic causal discovery experiments: it achieved strong conditional endpoint recovery while remaining the fastest non-ablation method.}
\label{table:cyto_causal_discovery}
\end{table}

\section{Conclusion}

We proposed BLITZ, a CI test based on two-stage
residualization. In its fastest implementation, BLITZ pairs polynomial
regression with decision-tree residualization. This pairing is computationally
attractive because the polynomial stage removes broad smooth structure that
would otherwise require many tree leaves to approximate. For
\(\beta\)-Hölder second-stage targets on an \(s\)-dimensional conditioning
space, where \(s=\dim(\bm Z)\) and \(\beta\) is the Hölder smoothness exponent, our
analysis shows that if polynomial preprocessing reduces the effective Hölder
seminorm by a factor \(\rho\), then the sufficient product leaf budget for the
two residualization sides is reduced by \(\rho^{2s/\beta}\). Since
\(\rho\ll 1\) in many smooth settings, this can yield a dramatic reduction in
the required tree size, allowing the second-stage trees to remain shallow and
fast.

Empirically, BLITZ completes in approximately the same amount of time as RCIT
and RCoT. This is notable because BLITZ remains calibrated, whereas RCIT and
RCoT obtain much of their speed from approximations that can lead to
miscalibrated \(p\)-values. Across null and alternative simulations, BLITZ
controlled the Type I error rate while maintaining competitive statistical
power. When embedded in constraint-based causal discovery algorithms, BLITZ was
especially strong at endpoint discovery conditional on the recovered skeleton.
This suggests that CI-test calibration matters not only for edge deletion, but
also for the separating sets recorded during skeleton search. By accepting
valid conditional independencies at more appropriate conditioning orders, BLITZ
improves the downstream orientation information available to PC-,
FCI-, and RFCI-type algorithms. BLITZ also often reduced the total runtime of
the causal discovery algorithms relative to competing tests. Taken together,
these results suggest that BLITZ is a calibrated, powerful, and fast
CI test designed for the computational and statistical
demands of causal discovery.

The two-stage residualization machinery underlying BLITZ may also prove useful beyond purely constraint-based causal discovery. Fast, calibrated CI tests could help screen candidate parent sets, define local lack-of-fit diagnostics, or restrict the search space before applying score-based or hybrid causal discovery methods. These extensions are not immediate, however, because score-based methods optimize global graph objectives rather than individual CI decisions \cite{Chickering02}. In addition, existing nonparametric score-based approaches evaluate candidate graphs using Gaussian-style likelihoods in kernel feature spaces together with fixed-fold cross-validation or related practical tuning rules, whose finite-sample and asymptotic behavior can be difficult to justify theoretically \cite{Huang18,Ren25}. Consequently, calibration at the level of individual CI tests does not by itself yield a calibrated or consistent graph score. Additional theory would therefore be needed to determine how BLITZ-style residualization should be incorporated into global graph search.

BLITZ also carries several limitations. First, our theoretical analysis is stated for \(\beta\)-Hölder second-stage
targets on an \(s\)-dimensional conditioning space, and therefore does not
directly cover highly irregular, discontinuous, or adversarial conditional mean
structures. Moreover, because the tree approximation rate scales as
\(K^{-\beta/s}\), the advantage of the method can deteriorate as the dimension
of the conditioning set increases. This limitation is usually mitigated in
constraint-based causal discovery, where many CI queries involve relatively
small conditioning sets, especially at early stages of skeleton search.
Nevertheless, BLITZ should not be interpreted as a general solution to
high-dimensional nonparametric CI testing.

Second, BLITZ benefits most when the polynomial stage removes a meaningful
amount of smooth, globally approximable conditional structure. If the remaining
dependence is entirely local or poorly represented by the residual feature
dictionary, the complexity reduction factor \(\rho\) may be close to one and
power may be reduced against such alternatives. These cases are possible but
somewhat adversarial: the dependence must both evade the polynomial stage and
be weakly represented by the chosen residual transformations. Finally, BLITZ
uses a moment-matched weighted chi-square approximation after data-adaptive
residualization, so calibration is expected to be more reliable in moderate to
large samples than in very small samples. This is a practical rather than
unexpected limitation, since fully nonparametric CI testing is unlikely to
offer much advantage over simpler parametric tests in small samples. Overall, future
work should extend the theory to broader function classes, improve robustness
to larger conditioning sets, and develop richer but still fast residual feature
dictionaries and null approximations.

We conclude that BLITZ provides a calibrated and computationally efficient CI test for repeated use in constraint-based causal discovery. Its two-stage residualization strategy enables strong Type I error control, competitive power, and improved endpoint discovery without sacrificing the speed needed for large numbers of CI queries.

\section*{Author's statements}

\subsection*{Funding information}

No funding was received for this study.

\subsection*{Author contributions}

Eric V. Strobl: conceptualization, methodology, software, formal analysis, investigation, writing--original draft, and writing--review and editing.

\subsection*{Conflict of interest}

The author declares no conflict of interest.

\subsection*{Ethical approval}

Not applicable.

\subsection*{Informed consent}

Not applicable.

\subsection*{Data availability statement}

The CYTO dataset analyzed in this article is publicly available from the source cited in the manuscript.

\bibliographystyle{plain}
\bibliography{biblio}

\section{Supplementary Materials}

\subsection{Computational Complexity} \label{supp:complexity}

Let \(n\) denote the sample size and let \(s=\dim(\bm Z)\) denote the number of nonconstant conditioning variables. BLITZ maps each scalar first-stage residual to two nonlinear features, so the transformed \(X\)- and \(Y\)-side feature dimensions are both equal to two. Consequently, the residual cross-covariance matrix has size \(2\times 2\), and the product-feature covariance matrix used for the null approximation has dimension \(4\times 4\).

Let \(r\) be the polynomial degree and let \(D_r\) be the number of polynomial basis functions in \(\bm Z\), including the intercept. With all interactions through degree \(r\), \(D_r=\binom{s+r}{r}\). In the default degree-two implementation, \(D_2=1+2s+\binom{s}{2}=O(s^2)\). The polynomial stage forms the normal equations, solves the resulting \(D_r\times D_r\) ridge-regularized least-squares problem, and computes residuals for both \(X\) and \(Y\). Since \(X\) and \(Y\) are scalar, this stage costs \(
O(nD_r^2+D_r^3).
\) For the default degree-two polynomial basis, this becomes \(O(ns^4+s^6)\).

The second stage fits shallow regression trees to remove remaining dependence of the transformed features on \(\bm Z\). For one tree fit, the implementation first presorts the active training rows along each conditioning variable, then recursively evaluates candidate splits and repartitions the sorted row lists. If \(s\) variables are evaluated at each split and \(H\) is the fitted or capped tree depth, the cost of one tree fit and prediction step is
\(
T_{\rm tree}
=
O(sn\log n+snH).
\)
The \(sn\log n\) term comes from presorting the conditioning variables. The \(snH\) term accounts for split evaluation over the selected variables, repartitioning of the sorted lists, and prediction through the fitted tree.

Let \(M=|\mathcal G|\) denote the number of candidate tree-size parameters. In the scalar case, the \(X\)-side transformed feature matrix has two columns and the \(Y\)-side transformed feature matrix has two columns. If \(M>1\) and both transformed features are used for tuning, BLITZ fits \(4M\) trees during out-of-bag tuning and four additional trees for the final residualization step. Thus the second-stage cost is
\(
O((4M+4)T_{\rm tree}).
\)
If no tuning is performed because \(M=1\), this reduces to four final tree fits.

The remaining operations are lower order. Standardization and nonlinear feature construction cost \(O(n)\). Computing the \(2\times 2\) residual cross-covariance matrix costs \(O(n)\). The null approximation uses the eigenvalues of a \(4\times 4\) product-feature covariance matrix, whose construction costs \(O(n)\) and whose eigendecomposition has constant cost. Therefore, in the scalar case,
\[
T_{\rm BLITZ}
=
O(nD_r^2+D_r^3+(4M+4)T_{\rm tree}+n).
\]
Under the default degree-two polynomial basis, this simplifies to
\[
T_{\rm BLITZ}
=
O(ns^4+s^6+(4M+4)s n(\log n+H)+n).
\]
For the low- and moderate-dimensional conditioning sets targeted by BLITZ, \(s\), \(M\), and \(H\) are small relative to \(n\). In this regime, the per-query complexity is near-linear in the sample size, with the dominant scaling approximately \(O(n\log n)\). This near-linear behavior explains why BLITZ can be used repeatedly inside constraint-based causal discovery algorithms, where thousands of scalar CI queries may be required.

\subsection{Proofs} \label{supp:proofs}

\begin{reptheorem}{thm:residuals}
Let \((\bm X,\bm Y,\bm Z)\) be random variables with \(\bm X\in\mathbb R^{d_X}\) and
\(\bm Y\in\mathbb R^{d_Y}\). Let \(m_{\bm X}\) and \(m_{\bm Y}\) be any measurable functions of
\(\bm Z \), and define
\[
R_{\bm X} \coloneqq \bm X-m_{\bm X}(\bm Z), \qquad R_{\bm Y} \coloneqq \bm Y-m_{\bm Y}(\bm Z).
\]
Then
\[
\bm X\perp\!\!\!\perp \bm Y\mid \bm Z
\quad\Longleftrightarrow\quad
\operatorname{Cov}\{f(R_{\bm X}),g(R_{\bm Y})\mid \bm Z\}=0
\quad\text{a.s.}
\]
for all square-integrable measurable functions \(f\) and \(g\).
\end{reptheorem}

\begin{proof}
Fix \(\bm Z=z\). The map \(x\mapsto x-m_{\bm X}(z)\) is bijective, with inverse
\(r\mapsto r+m_{\bm X}(z)\). Thus, conditional on \(\bm Z \), replacing \(\bm X\) by
\(R_{\bm X}\) preserves exactly the same information. Equivalently,
\(
\sigma(\bm X,\bm Z)=\sigma(R_{\bm X},\bm Z).
\)
The same argument gives
\(
\sigma(\bm Y,\bm Z)=\sigma(R_{\bm Y},\bm Z).
\)
Therefore
\(
\bm X\perp\!\!\!\perp \bm Y\mid \bm Z
\quad\Longleftrightarrow\quad
R_{\bm X}\perp\!\!\!\perp R_{\bm Y}\mid \bm Z.
\)

It remains to characterize \(R_{\bm X}\perp\!\!\!\perp R_{\bm Y}\mid \bm Z\). If
\(R_{\bm X}\perp\!\!\!\perp R_{\bm Y}\mid \bm Z\), then for all square-integrable measurable
\(f\) and \(g\) \cite{Dawid79},
\[
\mathbb E\{f(R_{\bm X})g(R_{\bm Y})\mid \bm Z\}
=
\mathbb E\{f(R_{\bm X})\mid \bm Z\}\,
\mathbb E\{g(R_{\bm Y})\mid \bm Z\}
 \quad\text{a.s.,}\]
which is exactly
\(
\operatorname{Cov}\{f(R_{\bm X}),g(R_{\bm Y})\mid \bm Z\}=0
\quad\text{a.s.}
\)

Conversely, suppose this conditional covariance vanishes for all
square-integrable measurable \(f\) and \(g\). Take
\(f=\mathbbm 1_A\) and \(g=\mathbbm 1_B\), for arbitrary measurable sets
\(A\subseteq\mathbb R^{d_X}\) and \(B\subseteq\mathbb R^{d_Y}\). Because \(\mathbb E\{f(R_{\bm X})g(R_{\bm Y})\mid \bm Z\}
=
\mathbb E\{f(R_{\bm X})\mid \bm Z\}\,
\mathbb E\{g(R_{\bm Y})\mid \bm Z\}
 \text{ a.s.}\), we have
\[
\mathbb P(R_{\bm X}\in A,R_{\bm Y}\in B\mid \bm Z)
=
\mathbb P(R_{\bm X}\in A\mid \bm Z)\,
\mathbb P(R_{\bm Y}\in B\mid \bm Z)
\quad\text{a.s.}
\]
Hence \(R_{\bm X}\perp\!\!\!\perp R_{\bm Y}\mid \bm Z\). Combining the two equivalences proves
the result.
\end{proof}

\begin{replemma}{lem:double}[Product bound for residual covariance bias]
Let \(\widehat{u}_X,\widehat{u}_Y\in\mathbb R^n\) denote the
second-stage BLITZ residuals for a chosen pair of transformed features, and let
\(\widehat C=n^{-1}\widehat{u}_X^\top\widehat{u}_Y\). Let
\(\mathcal Z_n=\sigma(\bm Z_1,\ldots,\bm Z_n)\), and define
\[
e_{X,n}
\coloneqq
\left(
\frac1n
\left\|
\mathbb E[\widehat{u}_X\mid \mathcal Z_n]
\right\|^2
\right)^{1/2},
\qquad
e_{Y,n}
\coloneqq
\left(
\frac1n
\left\|
\mathbb E[\widehat{u}_Y\mid \mathcal Z_n]
\right\|^2
\right)^{1/2}.
\]
Assume i.i.d. observations and \(H_0:X\perp\!\!\!\perp Y\mid \bm Z\). Then
\(\left|\mathbb E[\widehat C\mid \mathcal Z_n]\right|\le e_{X,n}e_{Y,n}\).
Consequently, if \(e_{X,n}e_{Y,n}\in o_p(n^{-1/2})\), then
\(\sqrt n\,\mathbb E[\widehat C\mid \mathcal Z_n]\to 0\) in probability. In
particular, this holds if \(e_{X,n}\in o_p(n^{-1/4})\) and
\(e_{Y,n}\in o_p(n^{-1/4})\).
\end{replemma}

\begin{proof}
Let \(\mu_X\coloneqq\mathbb E[\widehat{u}_X\mid \mathcal Z_n]\) and
\(\mu_Y\coloneqq\mathbb E[\widehat{u}_Y\mid \mathcal Z_n]\). Under i.i.d.
sampling and \(H_0\), \(X_{1:n}\perp\!\!\!\perp Y_{1:n}\mid \mathcal Z_n\).
By the side-specific construction in Algorithm~\ref{alg:blitz},
\(\widehat{u}_X\) is a function only of the \(X\)-side data and
\(\mathcal Z_n\), while \(\widehat{u}_Y\) is a function only of the
\(Y\)-side data and \(\mathcal Z_n\). Hence
\[
\mathbb E[\widehat C\mid \mathcal Z_n]
=
\frac1n \mu_X^\top\mu_Y .
\]
Therefore, by Cauchy--Schwarz,
\[
\left|
\mathbb E[\widehat C\mid \mathcal Z_n]
\right|
\le
\frac1n\|\mu_X\|\,\|\mu_Y\|
=
e_{X,n}e_{Y,n}.
\]
If \(e_{X,n}e_{Y,n}\in o_p(n^{-1/2})\), multiplying the bound by \(\sqrt n\)
gives
\(\sqrt n\,\left|\mathbb E[\widehat C\mid \mathcal Z_n]\right|\in o_p(1)\).
The balanced sufficient condition follows because
\(o_p(n^{-1/4})\,o_p(n^{-1/4})\in o_p(n^{-1/2})\).
\end{proof}

\begin{replemma}{lem:complexity_reduction}[Reduction in sufficient second-stage complexity]
Suppose Assumption~\ref{assump:standard_bound} holds. For a single side, write
\(\rho=\rho_q\). Under Assumption~\ref{assump:complexity_J}, the value of
\(K\) sufficient to approximate \(\eta_{q,m_q^\star}\) to tolerance
\(\tau>0\) is smaller than the value of \(K\) sufficient to approximate
\(\eta_{q,0}\) to the same tolerance by the factor \(\rho^{1/\alpha}\).
\end{replemma}

\begin{proof}
The approximation bound implies that tolerance \(\tau\) is achieved whenever
\(C\,\mathcal J(\eta)K^{-\alpha}\le \tau\), or equivalently whenever
\(K\ge (C\,\mathcal J(\eta)/\tau)^{1/\alpha}\).

Without first-stage residualization, a sufficient bound is therefore
\(K_0\ge (C\,\mathcal J(\eta_{q,0})/\tau)^{1/\alpha}\). For the first-stage
residualized target, Assumption~\ref{assump:complexity_J} gives
\(\mathcal J(\eta_{q,m_q^\star})\le \rho\,\mathcal J(\eta_{q,0})\). Hence it
is sufficient to take
\[
K_m
\ge
\left(
\frac{C\,\rho\,\mathcal J(\eta_{q,0})}{\tau}
\right)^{1/\alpha}
=
\rho^{1/\alpha}
\left(
\frac{C\,\mathcal J(\eta_{q,0})}{\tau}
\right)^{1/\alpha}.
\]
Thus the sufficient upper bound for the residualized target is smaller by the
factor \(\rho^{1/\alpha}\) relative to the corresponding sufficient bound for
the unreduced target.
\end{proof}

\begin{repproposition}{prop:product_complexity}[Product-complexity reduction for covariance bias]
Suppose the second-stage errors satisfy
\[
e_{X,n}\lesssim C_XJ_XK_X^{-\alpha},
\qquad
e_{Y,n}\lesssim C_YJ_YK_Y^{-\alpha}.
\]
If the first stage reduces the second-stage complexity constants by factors
\(\rho_X\) and \(\rho_Y\), then, at fixed \(K_X\) and \(K_Y\), the
covariance-bias bound is reduced by \(\rho_X\rho_Y\). Equivalently, to reach
the same covariance-bias tolerance, the sufficient product complexity
\(K_XK_Y\) is reduced by
\(
(\rho_X\rho_Y)^{1/\alpha}.
\)
In the symmetric case where \(\rho=\rho_X=\rho_Y\), this becomes
\(
\rho^{2/\alpha}.
\)
\end{repproposition}

\begin{proof}
The product bound gives
\(
e_{X,n}e_{Y,n}
\lesssim
C_XC_YJ_XJ_Y(K_XK_Y)^{-\alpha}.
\)
If the first stage replaces \(J_XJ_Y\) by
\(\rho_X\rho_YJ_XJ_Y\), then the covariance-bias constant is reduced by
\(\rho_X\rho_Y\). To make the product error at most \(\tau_C\), it is
sufficient that
\(
C_XC_YJ_XJ_Y(K_XK_Y)^{-\alpha}\le \tau_C.
\)
Solving for \(K_XK_Y\) gives a sufficient product complexity proportional to
\[
\left(
\frac{C_XC_YJ_XJ_Y}{\tau_C}
\right)^{1/\alpha}.
\]
After the first-stage reduction, this expression is multiplied by
\((\rho_X\rho_Y)^{1/\alpha}\). The symmetric case follows
immediately.
\end{proof}

\begin{replemma}{lem:trees}[Piecewise-constant approximation of H\"older functions]
Let \(\mathcal Z\subset\mathbb R^s\) be a bounded rectangle, and let
\(h:\mathcal Z\to\mathbb R\) be \(\beta\)-H\"older, \(0<\beta\le 1\), with
seminorm
\[
[h]_{\beta,\mathcal Z}
\coloneqq
\sup_{\substack{z,z'\in\mathcal Z\\ z\neq z'}}
\frac{|h(z)-h(z')|}{\|z-z'\|^\beta}.
\]
Let \(\mathcal T_K\) contain all axis-aligned piecewise-constant functions with
at most \(K\) rectangular cells. Then, for all sufficiently large \(K\),
\[
\inf_{t\in\mathcal T_K}\|h-t\|_{L_2(\mathbb P_{\bm Z})}
\le
C_{\mathcal Z,s,\beta}[h]_{\beta,\mathcal Z}K^{-\beta/s},
\]
where \(C_{\mathcal Z,s,\beta}<\infty\) depends on the domain, dimension,
and smoothness exponent.
\end{replemma}

\begin{proof}
Let \(\mathcal Z=\prod_{\ell=1}^s[a_\ell,b_\ell]\) and set
\(L\coloneqq\max_{1\le \ell\le s}(b_\ell-a_\ell)\). Let
\(M=\lfloor K^{1/s}\rfloor\). For \(K\) large enough, \(M\ge 1\). Partition
each coordinate interval \([a_\ell,b_\ell]\) into \(M\) equal subintervals, and
take the Cartesian products of these subintervals to form a grid of
\(M^s\le K\) axis-aligned rectangles covering \(\mathcal Z\). On each rectangle
\(Q\), choose a representative point \(z_Q\), and define \(t(z)=h(z_Q)\) for
\(z\in Q\). Each rectangle has diameter at most
\(s^{1/2}L/M\). Hence, for \(z\in Q\),
\[
|h(z)-t(z)|
=
|h(z)-h(z_Q)| \le [h]_{\beta,\mathcal Z} \|z - z_Q \|^\beta
\le
[h]_{\beta,\mathcal Z}
\left(\frac{s^{1/2}L}{M}\right)^\beta,
\]
where the first inequality follows from the seminorm. Therefore,
\[
\|h-t\|_{L_2(\mathbb P_{\bm Z})}
\le
\|h-t\|_\infty
\le
[h]_{\beta,\mathcal Z}s^{\beta/2}L^\beta M^{-\beta}.
\]
Since \(M=\lfloor K^{1/s}\rfloor\), there exists a constant
\(C_{\mathcal Z,s,\beta}<\infty\) such that
\(M^{-\beta}\le C_{\mathcal Z,s,\beta}K^{-\beta/s}\). Thus
\[
\inf_{t\in\mathcal T_K}\|h-t\|_{L_2(\mathbb P_{\bm Z})}
\le
C_{\mathcal Z,s,\beta}[h]_{\beta,\mathcal Z}K^{-\beta/s}.
\]
\end{proof}

\begin{reptheorem}{thm:poly_tree}[Polynomial--tree leaf-budget reduction]
Let \(\eta\) be a second-stage target and let \(h=\eta-p\) be the target left
after polynomial preprocessing. Suppose
\[
[h]_{\beta}\le \rho[\eta]_{\beta},
\qquad 0<\rho<1.
\]
Then the sufficient one-sided tree leaf budget needed to reach the same
approximation tolerance is reduced by
\(
\rho^{s/\beta}.
\)
If the same reduction factor \(\rho\) holds on both sides of the covariance
statistic, then the
sufficient product leaf budget \(K_XK_Y\) is reduced by
\(
\rho^{2s/\beta}.
\)
\end{reptheorem}
\begin{proof}
By Lemma~\ref{lem:trees}, the tree approximation bound has exponent \(\alpha=\beta/s\) and second-stage complexity constant \(\mathcal J(h)=[h]_{\beta}\). Lemma~\ref{lem:complexity_reduction} therefore gives a one-sided sufficient leaf-budget reduction of \(\rho^{1/\alpha}=\rho^{s/\beta}\). The product leaf-budget conclusion follows from Proposition~\ref{prop:product_complexity} with \(\alpha=\beta/s\).
\end{proof}

\begin{reptheorem}{thm:weighted_chisq}[Weighted chi-square null law for the BLITZ statistic]
Assume i.i.d. observations and \(H_0:X\perp\!\!\!\perp Y\mid \bm Z\). 
Assume that the vector-valued residual conditional-mean errors satisfy
\(e_{X,n}e_{Y,n}\in o_p(n^{-1/2})\).
Further suppose that the centered estimated product features satisfy
\[
\frac{1}{\sqrt n}\sum_{i=1}^{n}
\left[
\widehat{\bm W}_i-\mathbb{E}\!\left(\widehat{\bm W}_i\mid \mathcal Z_n\right)
\right]
\xrightarrow{d}
N(0,\Omega),
\]
where \(\Omega\in\mathbb R^{pq\times pq}\) is finite and positive
semidefinite. Then
\[
S
=
n\|\widehat C\|_F^2
=
\left\|
\frac{1}{\sqrt n}\sum_{i=1}^{n}\widehat{\bm W}_i
\right\|^2
\xrightarrow{d}
\sum_{r=1}^{pq}\lambda_r\chi_{1,r}^{2},
\]
where \(\lambda_1,\ldots,\lambda_{pq}\) are the eigenvalues of \(\Omega\), and
\(\chi_{1,1}^{2},\ldots,\chi_{1,pq}^{2}\) are independent chi-square random
variables with one degree of freedom.
\end{reptheorem}

\begin{proof}
By definition,
\[
\widehat C
=
\frac{1}{n}\widehat{\bm {U}}_X^{\top}\widehat{\bm {U}}_Y
=
\frac{1}{n}\sum_{i=1}^{n}
\widehat{\bm {U}}_{X,i}\widehat{\bm {U}}_{Y,i}^{\top}.
\]
Vectorizing both sides gives
\[
\operatorname{vec}(\widehat C)
=
\frac{1}{n}\sum_{i=1}^{n}
\operatorname{vec}\!\left(
\widehat{\bm {U}}_{X,i}\widehat{\bm {U}}_{Y,i}^{\top}
\right)
=
\frac{1}{n}\sum_{i=1}^{n}\widehat{\bm W}_i.
\]
Therefore,
\[
S
=
n\|\widehat C\|_F^2
=
n\|\operatorname{vec}(\widehat C)\|^2
=
\left\|
\frac{1}{\sqrt n}\sum_{i=1}^{n}\widehat{\bm W}_i
\right\|^2.
\]

We next show that the conditional-mean contribution is negligible. Let
\(\widehat\mu_{X,i}\coloneqq
\mathbb E(\widehat{\bm {U}}_{X,i}\mid \mathcal Z_n)\) and
\(\widehat\mu_{Y,i}\coloneqq
\mathbb E(\widehat{\bm {U}}_{Y,i}\mid \mathcal Z_n)\). Under i.i.d. sampling and the null,
\(X_{1:n}\perp\!\!\!\perp Y_{1:n}\mid \mathcal Z_n\). By the side-specific
construction in Algorithm~\ref{alg:blitz}, \(\widehat{\bm U}_X\) is a
measurable function only of the \(X\)-side data and \(\mathcal Z_n\), while
\(\widehat{\bm U}_Y\) is a measurable function only of the \(Y\)-side data and
\(\mathcal Z_n\). Hence, \(\widehat{\bm {U}}_X\perp\!\!\!\perp \widehat{\bm {U}}_Y\mid
\mathcal Z_n\). Hence
\[
\mathbb E\!\left(\widehat{\bm W}_i\mid \mathcal Z_n\right)
=
\operatorname{vec}\!\left(
\widehat\mu_{X,i}\widehat\mu_{Y,i}^{\top}
\right).
\]
It follows that
\[
\left\|
\frac{1}{\sqrt n}\sum_{i=1}^{n}
\mathbb E\!\left(\widehat{\bm W}_i\mid \mathcal Z_n\right)
\right\|
=
\frac{1}{\sqrt n}
\left\|
\sum_{i=1}^{n}
\widehat\mu_{X,i}\widehat\mu_{Y,i}^{\top}
\right\|_F.
\]
By Cauchy--Schwarz,
\[
\left\|
\sum_{i=1}^{n}
\widehat\mu_{X,i}\widehat\mu_{Y,i}^{\top}
\right\|_F
\le
\left\{
\sum_{i=1}^{n}\|\widehat\mu_{X,i}\|^2
\right\}^{1/2}
\left\{
\sum_{i=1}^{n}\|\widehat\mu_{Y,i}\|^2
\right\}^{1/2}.
\]
Using \(e_{X,n}e_{Y,n}\in o_p(n^{-1/2})\),
\[
\left\|
\frac{1}{\sqrt n}\sum_{i=1}^{n}
\mathbb E\!\left(\widehat{\bm W}_i\mid \mathcal Z_n\right)
\right\|
\le
\sqrt n\,e_{X,n}e_{Y,n}
=
o_p(1).
\]

Now decompose the normalized estimated product-feature sum:
\[
\frac{1}{\sqrt n}\sum_{i=1}^{n}\widehat{\bm W}_i
=
\frac{1}{\sqrt n}\sum_{i=1}^{n}
\left[
\widehat{\bm W}_i-\mathbb E\!\left(\widehat{\bm W}_i\mid \mathcal Z_n\right)
\right]
+
\frac{1}{\sqrt n}\sum_{i=1}^{n}
\mathbb E\!\left(\widehat{\bm W}_i\mid \mathcal Z_n\right).
\]
The first term converges in distribution to \(N(0,\Omega)\) by assumption, and
the second term is \(o_p(1)\). Hence, by Slutsky's theorem,
\[
\frac{1}{\sqrt n}\sum_{i=1}^{n}\widehat{\bm W}_i
\xrightarrow{d}
G,
\qquad
G\sim N(0,\Omega).
\]
By the continuous mapping theorem,
\[
S
=
\left\|
\frac{1}{\sqrt n}\sum_{i=1}^{n}\widehat{\bm W}_i
\right\|^2
\xrightarrow{d}
\|G\|^2.
\]
Let \(\Omega=Q\Lambda Q^\top\), where \(Q\) is orthogonal and
\(\Lambda=\operatorname{diag}(\lambda_1,\ldots,\lambda_{pq})\). If
\(\xi\sim N(0,I_{pq})\), then \(G\) has the same distribution as
\(Q\Lambda^{1/2}\xi\). Therefore
\[
\|G\|^2
=
\xi^\top\Lambda\xi
=
\sum_{r=1}^{pq}\lambda_r\xi_r^2
=
\sum_{r=1}^{pq}\lambda_r\chi_{1,r}^{2}.
\]
\end{proof}

\newpage
\subsection{Additional Empirical Test Results}

\begin{figure}[!h]
    \centering
    \includegraphics[width=1\linewidth]{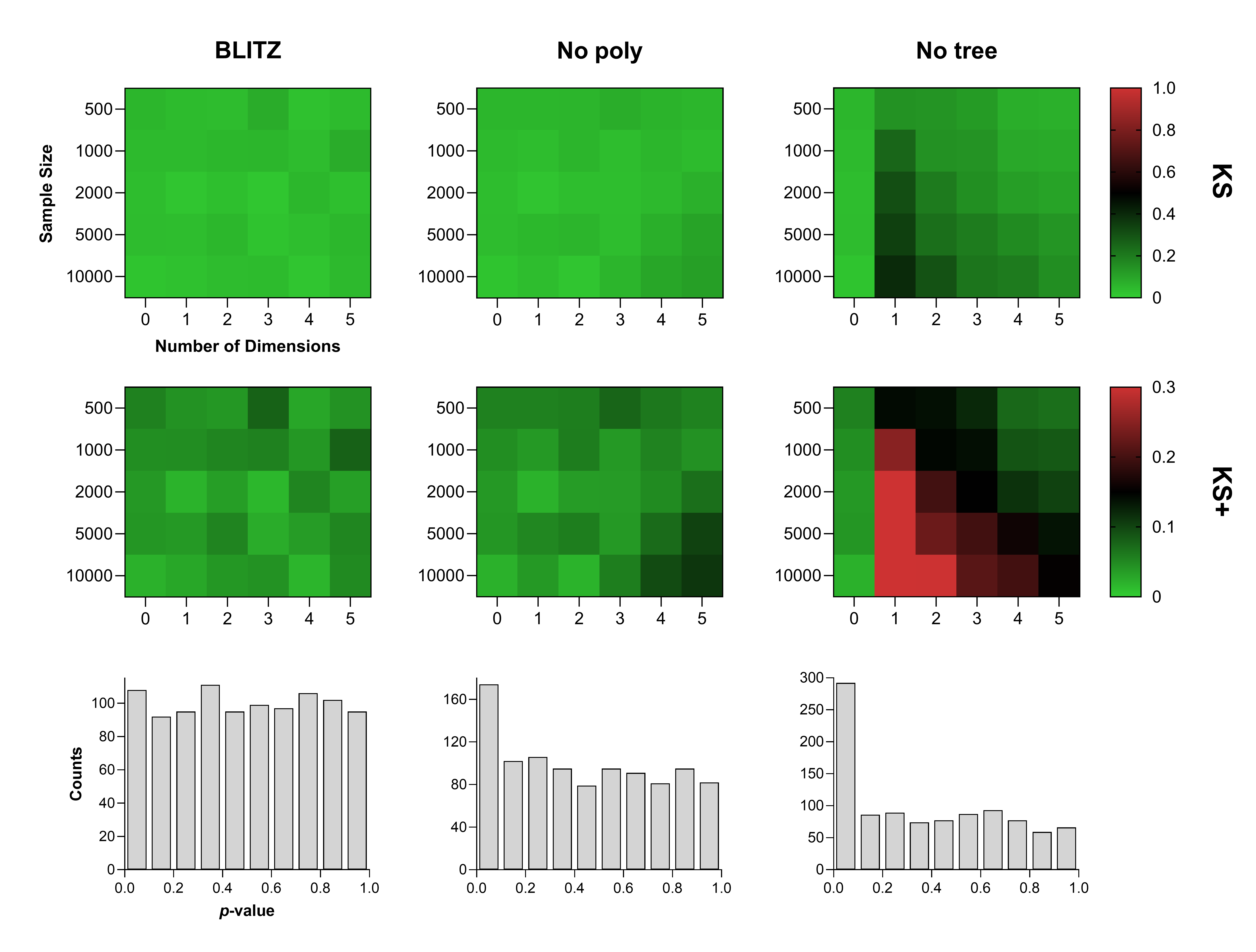}
    \caption{
    \textbf{Type I error ablations.} All heatmaps share the same x- and y-axis labels; likewise, all histograms use the same axis labels. Each column corresponds to one CI test. BLITZ provided the best overall Type I error control as measured by the KS statistic (first row). Although the no-polynomial ablation was close to BLITZ by the ordinary KS statistic, a more targeted analysis using the positive KS statistic, which measures anti-conservativeness through excess low \(p\)-values, showed that BLITZ better approximated the null distribution as sample size increased, whereas the no-polynomial ablation became more anti-conservative (second row). The third row shows representative null \(p\)-value histograms for one of the most challenging settings, \(n=10000\) and \(s=5\), illustrating the distributions summarized by the KS statistics.}
    \label{fig:Type1:ablations}
\end{figure}

\begin{figure}[!h]
    \centering
    \includegraphics[width=1\linewidth]{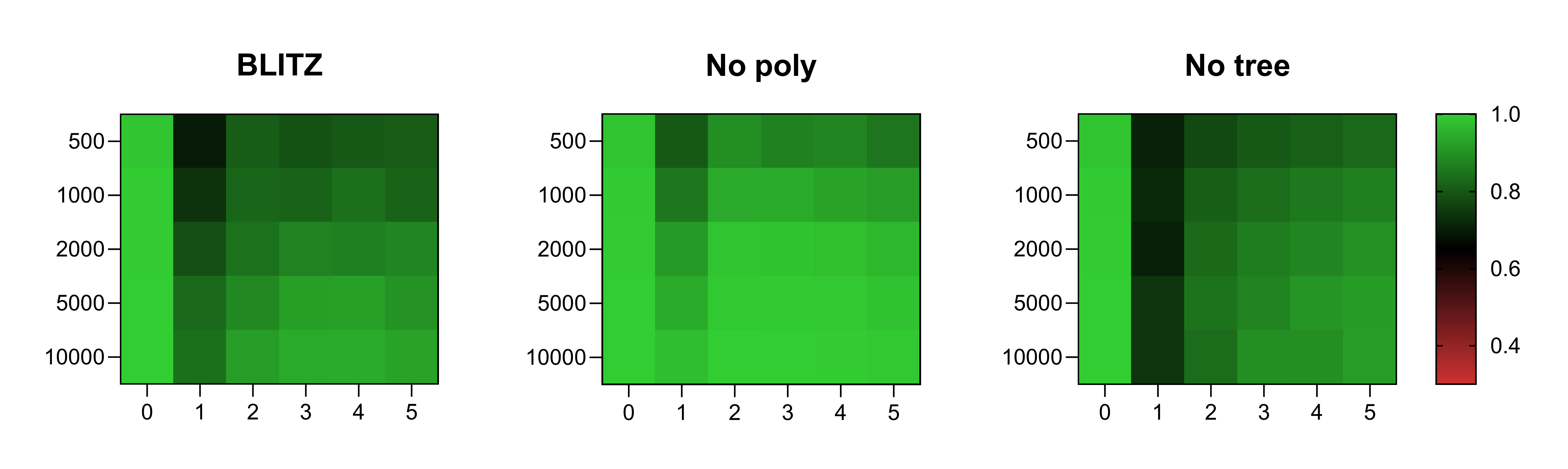}
    \caption{
    \textbf{Power ablations.} BLITZ achieved AUCPC competitive with both single-stage regression variants. As expected, using two residualization stages can modestly reduce power relative to single-stage variants that remove less \(\bm Z\)-dependent structure. In this experiment, the no-polynomial ablation achieved higher AUCPC, but this came at the cost of poorer Type I error calibration (Figure \ref{fig:Type1:ablations}).
    }\label{fig:Type2:ablations}
\end{figure}

\begin{figure}[!h]
    \centering
    \includegraphics[width=1\linewidth]{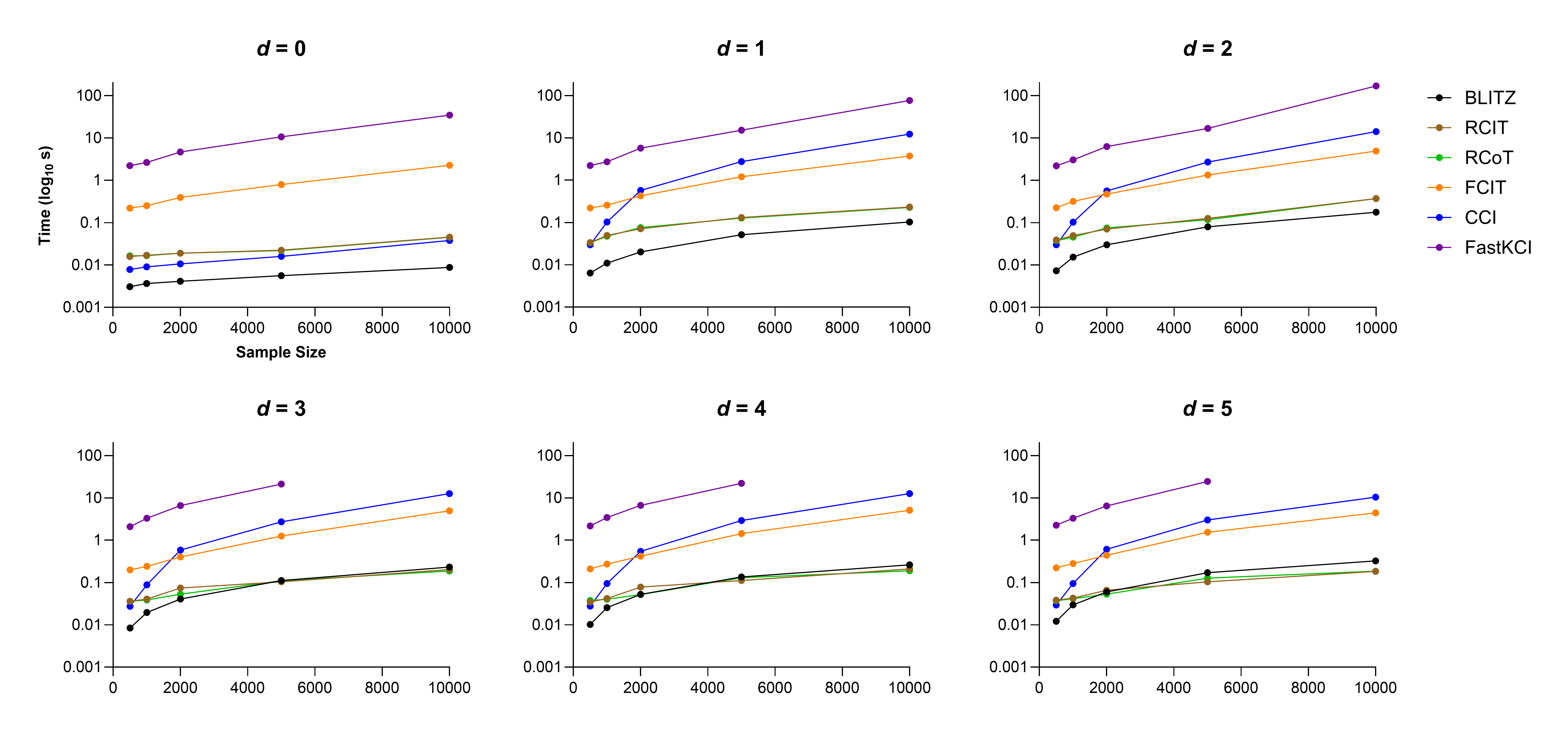}
    \caption{
    \textbf{Runtime under the alternative.} Mean wall-clock runtime for each CI test under the post-nonlinear alternative models. Results mimic those observed in Figure \ref{fig:Type1:time}.
    }
    \label{fig:Type2:time}
\end{figure}

\newpage
\subsection{Additional Causal Discovery and Real Data Results}

\begin{table}[!h]
\centering
\begin{tabular}{lccc}
\hline
Algorithm & Test & Precision & Recall \\ \midrule
\multirow{7}{*}{PC} & BLITZ & 0.786 [0.738, 0.833] & 0.794 [0.747, 0.841] \\
 & RCIT & 0.781 [0.735, 0.827] & 0.783 [0.737, 0.828] \\
 & -poly & 0.778 [0.735, 0.821] & 0.783 [0.741, 0.825] \\
 & RCoT & 0.769 [0.727, 0.810] & 0.771 [0.730, 0.813] \\
 & FCIT & 0.717 [0.668, 0.766] & 0.744 [0.696, 0.792] \\
 & CCI & 0.707 [0.666, 0.748] & 0.771 [0.732, 0.810] \\
 & -tree & 0.698 [0.646, 0.749] & 0.781 [0.735, 0.827] \\
\hline
\multirow{6}{*}{FCI} & RCIT & 0.731 [0.676, 0.785] & 0.731 [0.676, 0.785] \\
 & BLITZ & 0.701 [0.644, 0.759] & 0.704 [0.646, 0.762] \\
 & FCIT & 0.697 [0.633, 0.761] & 0.714 [0.649, 0.780] \\
 & RCoT & 0.684 [0.628, 0.740] & 0.685 [0.629, 0.741] \\
 & -poly & 0.670 [0.609, 0.731] & 0.674 [0.612, 0.735] \\
 & -tree & 0.564 [0.507, 0.621] & 0.606 [0.544, 0.667] \\
\hline
\multirow{7}{*}{RFCI} & FCIT & 0.693 [0.631, 0.755] & 0.705 [0.644, 0.766] \\
 & BLITZ & 0.617 [0.556, 0.677] & 0.623 [0.561, 0.685] \\
 & -poly & 0.603 [0.542, 0.665] & 0.610 [0.548, 0.672] \\
 & RCoT & 0.574 [0.509, 0.638] & 0.576 [0.511, 0.641] \\
 & RCIT & 0.572 [0.510, 0.633] & 0.573 [0.512, 0.635] \\
 & CCI & 0.523 [0.463, 0.584] & 0.548 [0.487, 0.609] \\
 & -tree & 0.511 [0.454, 0.567] & 0.567 [0.505, 0.628] \\
\hline
\multirow{7}{*}{Avg Rank} & BLITZ & \textbf{1.67} & \textbf{2.00} \\
 & RCIT & \underline{2.67} & \underline{3.00} \\
 & FCIT & 3.00 & 3.33 \\
 & -poly & 3.67 & 3.33 \\
 & RCoT & 4.00 & 4.33 \\
 & CCI & 6.00 & 6.50 \\
 & -tree & 6.67 & 5.33 \\
\hline
\end{tabular}
\caption{\textbf{Conditional endpoint precision and recall.} BLITZ achieved the best average rank for both conditional endpoint precision and conditional endpoint recall across the available causal discovery algorithms.}
\label{table:conditional_endpoint_precision_recall}
\end{table}

\begin{sidewaystable}
\centering
\begin{threeparttable}
\begin{tabular}{lccccccc}
\hline
Algorithm & Test & Adj Precision & Adj Recall & Adj F1 & Endpoint Precision & Endpoint Recall & Endpoint F1 \\ \midrule
\multirow{7}{*}{PC} & RCIT & 0.997 [0.993, 1.000] & 0.801 [0.767, 0.834] & 0.883 [0.860, 0.906] & 0.781 [0.735, 0.827] & 0.633 [0.586, 0.679] & 0.696 [0.650, 0.741] \\
 & RCoT & 0.996 [0.992, 1.000] & 0.818 [0.788, 0.848] & 0.894 [0.874, 0.915] & 0.769 [0.727, 0.810] & 0.633 [0.588, 0.677] & 0.691 [0.648, 0.734] \\
 & -poly & 0.993 [0.986, 0.999] & 0.806 [0.770, 0.843] & 0.883 [0.858, 0.909] & 0.778 [0.735, 0.821] & 0.633 [0.587, 0.678] & 0.693 [0.650, 0.736] \\
 & BLITZ & 0.988 [0.980, 0.997] & 0.770 [0.733, 0.807] & 0.858 [0.831, 0.885] & 0.786 [0.738, 0.833] & 0.613 [0.564, 0.662] & 0.683 [0.636, 0.730] \\
 & FCIT & 0.961 [0.940, 0.982] & 0.439 [0.399, 0.479] & 0.589 [0.549, 0.629] & 0.717 [0.668, 0.766] & 0.327 [0.290, 0.363] & 0.439 [0.398, 0.480] \\
 & CCI & 0.916 [0.893, 0.939] & 0.825 [0.790, 0.861] & 0.862 [0.836, 0.887] & 0.707 [0.666, 0.748] & 0.636 [0.593, 0.679] & 0.664 [0.625, 0.703] \\
 & -tree & 0.884 [0.857, 0.910] & 0.868 [0.836, 0.899] & 0.871 [0.847, 0.895] & 0.698 [0.646, 0.749] & 0.681 [0.633, 0.728] & 0.686 [0.638, 0.733] \\
\hline
\multirow{6}{*}{FCI} & RCIT & 1.000 [1.000, 1.000] & 0.471 [0.413, 0.529] & 0.614 [0.560, 0.669] & 0.731 [0.676, 0.785] & 0.359 [0.300, 0.418] & 0.462 [0.401, 0.524] \\
 & RCoT & 0.998 [0.995, 1.000] & 0.546 [0.493, 0.598] & 0.687 [0.642, 0.732] & 0.684 [0.628, 0.740] & 0.384 [0.330, 0.438] & 0.479 [0.424, 0.535] \\
 & BLITZ & 0.998 [0.992, 1.003] & 0.543 [0.492, 0.594] & 0.686 [0.642, 0.729] & 0.701 [0.644, 0.759] & 0.394 [0.339, 0.449] & 0.493 [0.436, 0.551] \\
 & -poly & 0.994 [0.984, 1.003] & 0.546 [0.491, 0.601] & 0.685 [0.637, 0.732] & 0.670 [0.609, 0.731] & 0.376 [0.321, 0.431] & 0.468 [0.410, 0.526] \\
 & FCIT & 0.979 [0.961, 0.998] & 0.279 [0.239, 0.319] & 0.418 [0.371, 0.464] & 0.697 [0.633, 0.761] & 0.207 [0.170, 0.244] & 0.307 [0.260, 0.354] \\
 & -tree & 0.937 [0.913, 0.962] & 0.713 [0.665, 0.762] & 0.797 [0.762, 0.831] & 0.564 [0.507, 0.621] & 0.444 [0.385, 0.504] & 0.489 [0.432, 0.546] \\
\hline
\multirow{7}{*}{RFCI} & RCIT & 0.996 [0.991, 1.000] & 0.686 [0.633, 0.738] & 0.797 [0.757, 0.838] & 0.572 [0.510, 0.633] & 0.408 [0.347, 0.469] & 0.468 [0.406, 0.530] \\
 & RCoT & 0.996 [0.991, 1.000] & 0.689 [0.646, 0.732] & 0.805 [0.775, 0.836] & 0.574 [0.509, 0.638] & 0.412 [0.352, 0.473] & 0.475 [0.412, 0.538] \\
 & BLITZ & 0.991 [0.984, 0.999] & 0.623 [0.576, 0.671] & 0.752 [0.714, 0.790] & 0.617 [0.556, 0.677] & 0.400 [0.344, 0.456] & 0.478 [0.420, 0.535] \\
 & -poly & 0.989 [0.981, 0.997] & 0.639 [0.588, 0.690] & 0.761 [0.721, 0.802] & 0.603 [0.542, 0.665] & 0.399 [0.344, 0.455] & 0.471 [0.414, 0.529] \\
 & FCIT & 0.981 [0.965, 0.998] & 0.314 [0.275, 0.354] & 0.461 [0.417, 0.505] & 0.693 [0.631, 0.755] & 0.221 [0.188, 0.254] & 0.325 [0.285, 0.365] \\
 & CCI & 0.952 [0.935, 0.968] & 0.688 [0.632, 0.744] & 0.781 [0.741, 0.821] & 0.523 [0.463, 0.584] & 0.388 [0.330, 0.446] & 0.436 [0.379, 0.493] \\
 & -tree & 0.906 [0.878, 0.933] & 0.749 [0.703, 0.794] & 0.808 [0.777, 0.839] & 0.511 [0.454, 0.567] & 0.434 [0.376, 0.492] & 0.463 [0.407, 0.519] \\
\hline
\multirow{7}{*}{Avg Rank} & RCIT & \textbf{1.00} & 4.67 & 3.67 & \underline{2.67} & 4.33 & 3.33 \\
 & RCoT & \underline{2.00} & 2.67 & \textbf{1.67} & 4.00 & \underline{3.00} & \underline{2.67} \\
 & BLITZ & 3.33 & 5.33 & 5.00 & \textbf{1.67} & 4.00 & \textbf{2.33} \\
 & -poly & 3.67 & 3.67 & 3.67 & 3.67 & 4.00 & 3.00 \\
 & FCIT & 5.00 & 6.67 & 6.67 & 3.00 & 6.67 & 6.67 \\
 & CCI & 6.00 & \underline{2.50} & 4.50 & 6.00 & 4.00 & 6.00 \\
 & -tree & 6.67 & \textbf{1.00} & \underline{2.00} & 6.67 & \textbf{1.00} & 3.67 \\
\hline
\end{tabular}
\caption{\textbf{Additional structural recovery metrics.} BLITZ did not primarily improve adjacency recovery, but it achieved the best average rank for exact endpoint precision and exact endpoint F1, indicating that its main advantage lies in the reliability of recovered endpoint orientations rather than in retaining more adjacencies.}
\label{table:additional_structural_metrics}
\end{threeparttable}
\end{sidewaystable}

\begin{table}[!h]
\centering
\begin{tabular}{lccc}
\hline
Algorithm & Test & Precision & Recall \\ \midrule
\multirow{7}{*}{PC} & FCIT & 0.811 [0.787, 0.835] & 0.870 [0.856, 0.883] \\
 & -poly & 0.799 [0.798, 0.800] & 0.888 [0.888, 0.889] \\
 & BLITZ & 0.798 [0.795, 0.801] & 0.888 [0.886, 0.889] \\
 & RCIT & 0.754 [0.728, 0.780] & 0.848 [0.823, 0.872] \\
 & CCI & 0.728 [0.696, 0.761] & 0.827 [0.800, 0.855] \\
 & RCoT & 0.682 [0.647, 0.716] & 0.798 [0.764, 0.831] \\
 & -tree & 0.552 [0.517, 0.588] & 0.680 [0.647, 0.713] \\
\hline
\multirow{7}{*}{FCI} & FCIT & 0.825 [0.808, 0.843] & 0.878 [0.870, 0.886] \\
 & BLITZ & 0.799 [0.796, 0.801] & 0.888 [0.887, 0.889] \\
 & -poly & 0.799 [0.796, 0.801] & 0.888 [0.887, 0.889] \\
 & RCoT & 0.797 [0.790, 0.803] & 0.889 [0.882, 0.895] \\
 & RCIT & 0.789 [0.776, 0.802] & 0.880 [0.870, 0.891] \\
 & CCI & 0.780 [0.759, 0.802] & 0.870 [0.848, 0.891] \\
 & -tree & 0.580 [0.541, 0.619] & 0.700 [0.662, 0.737] \\
\hline
\multirow{7}{*}{RFCI} & BLITZ & 0.799 [0.798, 0.800] & 0.888 [0.888, 0.889] \\
 & -poly & 0.799 [0.796, 0.801] & 0.888 [0.887, 0.889] \\
 & FCIT & 0.794 [0.768, 0.821] & 0.852 [0.835, 0.868] \\
 & RCIT & 0.767 [0.744, 0.790] & 0.863 [0.844, 0.882] \\
 & CCI & 0.718 [0.680, 0.756] & 0.812 [0.776, 0.848] \\
 & RCoT & 0.681 [0.639, 0.722] & 0.792 [0.752, 0.832] \\
 & -tree & 0.525 [0.484, 0.565] & 0.645 [0.604, 0.685] \\
\hline
\multirow{7}{*}{Avg Rank} & FCIT & \textbf{1.67} & 4.00 \\
 & BLITZ & \underline{2.17} & \textbf{1.83} \\
 & -poly & \underline{2.17} & \textbf{1.83} \\
 & RCIT & 4.33 & \underline{3.67} \\
 & CCI & 5.33 & 5.33 \\
 & RCoT & 5.33 & 4.33 \\
 & -tree & 7.00 & 7.00 \\
\hline
\end{tabular}
\caption{\textbf{Conditional endpoint precision and recall for CYTO.} BLITZ achieved the highest conditional endpoint recall and the second-highest conditional endpoint precision, indicating that its strong endpoint F1 reflected consistently strong performance on both components rather than a favorable precision--recall tradeoff.}
\label{table:cyto_conditional_endpoint_precision_recall}
\end{table}

\begin{sidewaystable}
\centering
\begin{threeparttable}
\begin{tabular}{lccccccc}
\hline
Algorithm & Test & Adj Precision & Adj Recall & Adj F1 & Endpoint Precision & Endpoint Recall & Endpoint F1 \\ \midrule
\multirow{7}{*}{PC} & FCIT & 0.931 [0.912, 0.950] & 0.420 [0.402, 0.438] & 0.576 [0.557, 0.594] & 0.811 [0.787, 0.835] & 0.366 [0.348, 0.383] & 0.501 [0.482, 0.520] \\
 & -poly & 0.900 [0.899, 0.900] & 0.498 [0.495, 0.501] & 0.641 [0.638, 0.644] & 0.799 [0.798, 0.800] & 0.442 [0.439, 0.445] & 0.569 [0.566, 0.572] \\
 & BLITZ & 0.899 [0.898, 0.900] & 0.496 [0.489, 0.502] & 0.639 [0.633, 0.645] & 0.798 [0.795, 0.801] & 0.440 [0.434, 0.446] & 0.567 [0.561, 0.573] \\
 & RCIT & 0.887 [0.879, 0.896] & 0.502 [0.497, 0.508] & 0.641 [0.636, 0.646] & 0.754 [0.728, 0.780] & 0.425 [0.414, 0.436] & 0.543 [0.527, 0.559] \\
 & CCI & 0.877 [0.862, 0.891] & 0.492 [0.476, 0.509] & 0.628 [0.610, 0.647] & 0.728 [0.696, 0.761] & 0.407 [0.388, 0.426] & 0.521 [0.496, 0.545] \\
 & RCoT & 0.851 [0.836, 0.866] & 0.509 [0.503, 0.515] & 0.636 [0.629, 0.643] & 0.682 [0.647, 0.716] & 0.405 [0.389, 0.421] & 0.508 [0.486, 0.529] \\
 & -tree & 0.806 [0.790, 0.822] & 0.509 [0.503, 0.515] & 0.623 [0.616, 0.630] & 0.552 [0.517, 0.588] & 0.346 [0.330, 0.361] & 0.425 [0.402, 0.447] \\
\hline
\multirow{7}{*}{FCI} & FCIT & 0.940 [0.923, 0.958] & 0.412 [0.394, 0.431] & 0.570 [0.551, 0.588] & 0.825 [0.808, 0.843] & 0.362 [0.345, 0.379] & 0.500 [0.483, 0.518] \\
 & BLITZ & 0.899 [0.898, 0.900] & 0.497 [0.492, 0.502] & 0.640 [0.635, 0.644] & 0.799 [0.796, 0.801] & 0.441 [0.436, 0.446] & 0.568 [0.563, 0.573] \\
 & -poly & 0.899 [0.898, 0.900] & 0.497 [0.492, 0.502] & 0.640 [0.635, 0.644] & 0.799 [0.796, 0.801] & 0.441 [0.436, 0.446] & 0.568 [0.563, 0.573] \\
 & RCoT & 0.897 [0.892, 0.901] & 0.499 [0.497, 0.501] & 0.641 [0.639, 0.643] & 0.797 [0.790, 0.803] & 0.443 [0.439, 0.447] & 0.570 [0.565, 0.574] \\
 & CCI & 0.896 [0.891, 0.901] & 0.498 [0.491, 0.505] & 0.640 [0.633, 0.647] & 0.780 [0.759, 0.802] & 0.433 [0.421, 0.445] & 0.556 [0.541, 0.572] \\
 & RCIT & 0.895 [0.890, 0.900] & 0.493 [0.488, 0.499] & 0.636 [0.631, 0.641] & 0.789 [0.776, 0.802] & 0.434 [0.427, 0.442] & 0.560 [0.551, 0.570] \\
 & -tree & 0.823 [0.809, 0.836] & 0.506 [0.501, 0.510] & 0.626 [0.620, 0.631] & 0.580 [0.541, 0.619] & 0.353 [0.335, 0.371] & 0.438 [0.413, 0.463] \\
\hline
\multirow{7}{*}{RFCI} & FCIT & 0.931 [0.910, 0.952] & 0.414 [0.398, 0.431] & 0.571 [0.553, 0.589] & 0.794 [0.768, 0.821] & 0.353 [0.337, 0.370] & 0.487 [0.468, 0.506] \\
 & BLITZ & 0.900 [0.899, 0.900] & 0.498 [0.495, 0.501] & 0.641 [0.638, 0.644] & 0.799 [0.798, 0.800] & 0.442 [0.439, 0.445] & 0.569 [0.566, 0.572] \\
 & -poly & 0.899 [0.898, 0.900] & 0.497 [0.492, 0.502] & 0.640 [0.635, 0.644] & 0.799 [0.796, 0.801] & 0.441 [0.436, 0.446] & 0.568 [0.563, 0.573] \\
 & RCIT & 0.887 [0.878, 0.895] & 0.499 [0.497, 0.501] & 0.638 [0.635, 0.641] & 0.767 [0.744, 0.790] & 0.431 [0.421, 0.440] & 0.551 [0.538, 0.565] \\
 & CCI & 0.879 [0.868, 0.891] & 0.499 [0.493, 0.505] & 0.636 [0.631, 0.641] & 0.718 [0.680, 0.756] & 0.404 [0.387, 0.422] & 0.517 [0.493, 0.541] \\
 & RCoT & 0.853 [0.836, 0.869] & 0.508 [0.501, 0.514] & 0.636 [0.628, 0.644] & 0.681 [0.639, 0.722] & 0.402 [0.382, 0.422] & 0.505 [0.478, 0.532] \\
 & -tree & 0.806 [0.790, 0.822] & 0.509 [0.503, 0.515] & 0.623 [0.616, 0.630] & 0.525 [0.484, 0.565] & 0.327 [0.308, 0.347] & 0.403 [0.376, 0.429] \\
\hline
\multirow{7}{*}{Avg Rank} & FCIT & \textbf{1.00} & 7.00 & 7.00 & \textbf{1.67} & 6.00 & 6.00 \\
 & BLITZ & \underline{2.50} & 4.83 & \textbf{2.17} & \underline{2.17} & \textbf{1.83} & \textbf{1.83} \\
 & -poly & \underline{2.50} & 4.83 & \textbf{2.17} & \underline{2.17} & \textbf{1.83} & \textbf{1.83} \\
 & RCIT & 4.67 & 4.17 & \underline{3.00} & 4.33 & \underline{3.33} & \underline{3.33} \\
 & CCI & 5.00 & 4.17 & 4.33 & 5.33 & 4.33 & 4.33 \\
 & RCoT & 5.33 & \underline{1.83} & 3.33 & 5.33 & 3.67 & 3.67 \\
 & -tree & 7.00 & \textbf{1.17} & 6.00 & 7.00 & 7.00 & 7.00 \\
\hline
\end{tabular}
\caption{\textbf{Additional CYTO structural recovery metrics.} As in the synthetic experiments, BLITZ also performed well on exact endpoint recovery, while remaining competitive on adjacency-level recovery. These results suggest that its CYTO performance was not limited to conditional endpoint metrics.}
\label{table:cyto_additional_structural_metrics}
\end{threeparttable}
\end{sidewaystable}

\end{document}